\documentclass{article}

\usepackage{geometry}

\usepackage[utf8]{inputenc} 
\usepackage[T1]{fontenc}    
\usepackage{hyperref}       
\usepackage{url}            
\usepackage{booktabs}       
\usepackage{amsfonts}       
\usepackage{nicefrac}       
\usepackage{microtype}      

\usepackage{subcaption}

\usepackage{amssymb}
\usepackage{amsmath}
\usepackage{amsthm}
\usepackage{mathtools}
\usepackage{amsfonts}
\usepackage{enumerate}
\usepackage{bbm}
\usepackage{hhline}
\usepackage{color}
\newcommand{\virag}[1]{{\color{blue}[Virag: #1]}}
\newcommand{\rj}[1]{{\color{red}[Ramesh: #1]}}
\newcommand{\kth}{^\text{th}}
\newcommand{\ceil}[1]{\left\lceil{#1}\right\rceil}
\newcommand{\floor}[1]{\left\lfloor{#1}\right\rfloor}
\newcommand{\Oracle}{\ensuremath{\mathsf{Oracle}}}
\renewcommand{\P}{\ensuremath{\mathbb{P}}}
\newcommand{\E}{\ensuremath{\mathbb{E}}}

\newtheorem{definition}{Definition}
\newtheorem{theorem}{Theorem}
\newtheorem{lemma}{Lemma}

\newtheorem{proposition}{Proposition}

\title{Bandit Learning with Positive Externalities}

\author{
  Virag Shah, 
  Jose Blanchet,
  Ramesh Johari\\
Stanford University
}

\begin{document}

\maketitle

\begin{abstract}
In many platforms, user arrivals exhibit a self-reinforcing behavior: future user arrivals are likely to have preferences similar to users who were satisfied in the past. In other words, arrivals exhibit {\em positive externalities}. We study multiarmed bandit (MAB) problems with positive externalities. We show that the self-reinforcing preferences may lead standard benchmark algorithms such as UCB to exhibit linear regret. We develop a new algorithm, Balanced Exploration (BE), which explores arms carefully to avoid suboptimal convergence of arrivals before sufficient evidence is gathered. We also introduce an adaptive variant of BE which successively eliminates suboptimal arms. We analyze their asymptotic regret, and establish optimality by showing that no algorithm can perform better.

\end{abstract}


\section{Introduction}
\label{sec:intro}

A number of different platforms use multiarmed bandit (MAB) algorithms today to optimize their service: e.g., search engines and information retrieval platforms; e-commerce platforms; and news sites.  Many such platforms exhibit a natural self-reinforcement in the arrival process of users: future arrivals may be biased towards users who expect to have positive experiences based on the past outcomes of the platform.  For example, if a news site generates articles that are liberal (resp., conservative), then it is most likely to attract additional users who are liberal (resp., conservative) \cite{adamic2005political}.  In this paper, we study the optimal design of MAB algorithms when user arrivals exhibit such positive self-reinforcement.

We consider a setting in which a platform faces many types of users that can arrive.  Each user type is distinguished by preferring a subset of the item types above all others.  The platform is not aware of either the type of the user, or the item-user payoffs.  Following the discussion above, arrivals exhibit {\em positive externalities} (also called positive network effects) among the users \cite{katz1994systems}: in particular, if one type of item generates positive rewards, users who prefer that type of item become more likely to arrive in the future.  


Our paper quantifies the consequences of positive externalities for bandit learning in a benchmark model where the platform is unable to observe the user's type on arrival.    In the model we consider, introduced in Section \ref{sec:problem}, there is a set of $m$ arms. A given arriving user prefers a subset of these arms over the others; in particular, all arms other than the preferred arms generate zero reward. A preferred arm $a$ generates a Bernoulli reward with mean $\mu_a$.  To capture positive externalities, the probability that a user preferring arm $a$ arrives at time $t$ is proportional to $(S_a(t-1) + \theta_a)^\alpha$, where $S_a(t-1)$ is the total reward observed from arm $a$ in the past and $\theta_a$ captures the initial conditions. The positive constant $\alpha$ captures the strength of the externality: when $\alpha$ is large the positive externality is strong.  

The platform aims to maximize cumulative reward up to time horizon $T$.  We evaluate our performance by measuring regret against an ``offline'' oracle that always chooses the arm $a^* = \arg \max_a \mu_a$.  Because of the positive externality, this choice causes the user population to shift entirely to users preferring arm $a^*$ over time; in particular, the oracle achieves asymptotically optimal performance to leading order in $T$.  We study the asymptotic scaling of cumulative regret against the oracle at $T$ as $T \to \infty$.  

At the heart of this learning problem is a central tradeoff.  On one hand, because of the positive externality, the platform operator is able to move the user population towards the profit maximizing population.  On the other hand, due to self-reinforcing preferences the impact of {\em mistakes} is amplified: if rewards are generated on suboptimal arms, the positive externality causes more users that prefer those arms to arrive in the future.  We are able to explicitly quantify the impact of this tradeoff in our model.

Our main results are as follows.

{\bf Lower bound}.  In Section \ref{sec:lower}, we provide an explicit lower bound on the best achievable regret  for each $\alpha$.  Strikingly, the optimal regret is structurally quite different than classical lower bounds for MAB problems; see Table \ref{tbl:performance}. Its development sheds light into the key differences between MABs with positive externalities and those without. 

{\bf Suboptimality of classical approaches}.  In Section \ref{sec:ucb}, we show that the UCB algorithm is not only suboptimal, but in fact has positive probability of {\em never} obtaining a reward on the best arm $a^*$---and thus obtains linear regret.  This is because UCB does not explore sufficiently to find the best arm. However, we show that just exploring more aggressively is also insufficient; a random-explore-then-commit policy which explores in an unstructured fashion remains suboptimal. This demonstrates the need of developing a new approach to exploration.

{\bf Optimal algorithm}.  In Section \ref{sec:waterfill}, we develop a new algorithmic approach towards optimizing the exploration-exploitation tradeoff.  Interestingly, this algorithm is cautious in the face of uncertainty to avoid making long-lasting mistakes. Our algorithm, Balanced Exploration (BE), keeps the user population ``balanced'' during the exploration phase; by doing so, it exploits an arm only when there is sufficient certainty regarding its optimality. Its adaptive variant, Balanced Exploration with Arm Elimination (BE-AE), intelligently eliminates suboptimal arms while balancing exploration among the remainder.  BE has the benefit of not depending on system parameters, while BE-AE uses such information (e.g., $\alpha$).  We establish their optimality by developing an upper-bound on their regret for each $\alpha$; this nearly matches the lower bound (for BE), and exactly matches the lower bound (for BE-AE).  

%
%

\begin{table*}
\begin{center}
\def\arraystretch{1.75}
\begin{tabular}{c||c|c|c|c|}
 & $\alpha = 0$ & $0< \alpha<1$ & $\alpha = 1$ & $\alpha >1$ \\ 
 \hhline{=#=|=|=|=|}
 Lower Bound	& $\Omega(\ln T)$ 	&  $\Omega(T^{1-\alpha} \ln^\alpha T)$	& $\Omega(\ln^2 T)$  & $\Omega(\ln^\alpha T)$ \\
\hline
UCB  		& $O(\ln T)$    	& $\Omega(T)$ 				& $\Omega(T)$	 &	$\Omega(T)$\\
\hline
\!\!Random-explore-then-commit		& $O(\ln T)$	& \!\! $\Omega\left(T^{1-\alpha} \ln ^{\frac{\alpha}{1-\alpha}} T\right)$ \!\! & \!\! $\Omega\left(T^{\frac{\mu_b}{\mu_b+\theta_{a^*}\mu_{a^*}}}\right)$ \!\! & $\Omega(T)$ \\
\hline
Balanced Exploration (BE)	& $\tilde{O}(\ln T)$	&  $\tilde{O}(T^{1-\alpha} \ln^{\alpha} T)$ 			& $\tilde{O}(\ln^2 T)$	&  $\tilde{O}(\ln^\alpha T)$\\
\hline
\!\! \!\!BE with Arm Elimination	(BE-AE) & $O(\ln T)$	&  $O(T^{1-\alpha} \ln^{\alpha} T)$ 			& $O(\ln^2 T)$	&  $O(\ln^\alpha T)$\\
\hline
\end{tabular}
\end{center}
\caption{Total regret under different settings. 
Here $a^* = \arg \max \mu_a$, and $b = \arg \max_{a\neq a^*} \mu_a$. For Random-explore-then-commit algorithm, we assume that the initial bias $\theta_a$ for each arm $a$ is a positive integer (cf.~Section \ref{sec:problem}). The notation $f(T) = \tilde{O}(g(T))$ implies there exists $k > 0$ such that $f(T) = O(g(T)\ln^k \!g(T))$.
\label{tbl:performance}
}
\end{table*}

Further, in Section~\ref{sec:simulations} we provide simulation results to obtain quantitative insights into the relative performance of different algorithms. We conclude the paper by summarizing the main qualitative insights obtained from our work.    

\section{Related work}

As noted above, our work incorporates positive externalities in user arrivals. Positive externalities are also referred to as positive network effects or positive network externalities.  (Note that the phrase ``network'' is often used here, even when the effects do not involve explicit network connections between the users.)  See \cite{katz1994systems}, as well as \cite{shy2011short, shapiro1998information} for background.  Positive externalities are extensively discussed in most standard textbooks on microeconomic theory; see, e.g., Chapter 11 of \cite{mas1995microeconomic}.

It is well accepted that online search and recommendation engines produce feedback loops that can lead to self-reinforcement of popular items \cite{anderson2004long, barabasi1999emergence, ratkiewicz2010characterizing, chakrabarti2005influence}.  Our model captures this phenomenon by employing a self-reinforcing arrival process, inspired by classical urn processes \cite{AtK68, Jan04}.

We note that the kind of self-reinforcing behavior observed in our model may be reminiscent of ``herding'' behavior in Bayesian social learning \cite{bikhchandani1992theory, SmS00, ADL11}.  In these models, arriving Bayesian rational users take actions based on their own private information, and the outcomes experienced by past users. 
The central question in that literature is the following: do individuals base their actions on their own private information, or do they follow the crowd? 
By contrast, in our model it is the platform which takes actions, without directly observing preferences of the users.

If the user preferences are known then a platform might choose to personalize its services to satisfy each user individually.  This is the theme of much recent work on {\em contextual bandits}; see, e.g., \cite{langford2008epoch, slivkins2011contextual, perchet2013multi} and \cite{Buc12} for a survey of early work. In such a model, it is important that either (1) enough observable covariates are available to group different users together as decisions are made; or (2) users are long-lived so that the platform has time to learn about them. 

In contrast to contextual bandits, in our model the users' types are not known, and they are short-lived (one interaction per user). Of course, the reality of many platforms is somewhere in between: some user information may be available, though imperfect.  We view our setting as a natural benchmark model for analysis of the impact of self-reinforcing arrivals.  Through this lens, our work suggests that there are significant consequences to learning when the user population itself can change over time, an insight that we expect to be robust across a wide range of settings.

\section{Preliminaries}
\label{sec:problem}

In this section we describe the key features of the model we study.  We first describe the model, including a precise description of the arrival process that captures positive externalities.  Next, we describe our objective: minimization of regret relative to the expected reward of a natural oracle policy.

\subsection{Model}

{\bf Arms and rewards.}  Let $A=\{1,...,m\}$ be the set of available arms. 
During each time $t\in \{1,2,...\}$ a new user arrives and an arm is ``pulled'' by the platform; we denote the arm pulled at time $t$ by $I_{t}$.  We view pulling an arm as presenting the corresponding option to the newly arrived user. Each arriving user prefers a subset of the arms, denoted by $J_{t}$.  We describe below how $J_t$ is determined.  

If arm $a$ is pulled at time $t$ and if the user at time $t$ prefers arm $a\in A$ (i.e., $a \in J_t$) then the reward obtained at time $t$ is an independent Bernoulli random variable with mean $\mu_{a}$. We assume $\mu_a > 0$ for all arms. If the user at time $t$ does not prefer the arm pulled then the reward obtained at time $t$ is zero.  We let $X_{t}$ denote the reward obtained at time $t$.

 For $t \geq 1$, let $T_a(t)$ represent the number of times arm $a$ is pulled up to and including time $t$, and let $S_a(t)$ represent the total reward accrued by pulling arm $a$ up to and including time $t \geq 1$.  Thus $T_a(t) = |\{ 1 \leq s \leq t : I_s = a\}|$, and $S_a(t) = |\{ 1 \leq s \leq t : I_s = a, X_s = 1\}|$.  We define $T_a(0) = S_a(0) = 0$.

{\bf Unique best arm}. We assume there exists a unique $a^* \in A$ such that:
\[ a^* = \arg \max \mu_a. \]
This assumption is standard and made for technical convenience; all our results continue to hold without it.  

{\bf Arrivals with positive externalities.}  We now define the arrival process $\{J_t\}_{t\ge 1}$ that determines users' preferences over arms; this arrival process is the novel feature of our model.  We assume there are fixed constants $\theta_{a}>0$ for $a\in A$ (independent of $T$), denoting the initial ``popularity'' of arm $a$.  

For $t \geq 0$, define: 
\[ N_a(t) = S_a(t) + \theta_a,\ \ a \in A. \]
Observe that by definition $N_a(0) = \theta_a$.  

In our arrival process, arms with higher values of $N_a(t)$ are more likely to be preferred.  Formally, we assume that the $t\kth$ user prefers arm $a$ (i.e., $a\in J_{t}$) with probability $\lambda_a(t)$ independently of other arms, where:
\[ \lambda_a(t) = \frac{f(N_a(t-1))}{\sum_{a' = 1}^m f(N_{a'}(t-1))}, \]
where $f(\cdot)$ is a positive, increasing function $f$. We refer to $f$ as the {\em externality function}. In our analysis we primarily focus on the parametric family $f(x) = x^\alpha$, where $\alpha \in (0,\infty)$.

Intuitively, the idea is that agents who prefer arm $a$ are more likely to arrive if arm $a$ has been successful in the past.  This is a positive externality: users who prefer arm $a$ are more likely to generate rewards when arm $a$ is pulled, and this will in turn increase the likelihood an arrival preferring arm $a$ comes in the future.  The parameter $\alpha$ controls the strength of this externality: the positive externality is stronger when $\alpha$ is larger.

If $f$ is linear ($\alpha = 1$), then we can interpret our model in terms of an urn process.
In this view, $\theta_a$ resembles the initial number of balls of color $a$ in the urn at time $t=1$ and $N_a(t)$ resembles the total number of balls of color $a$ added into the urn after $t$ draws. Thus, the probability the $t\kth$ draw is of color $a$ is proportional to $N_a(t)$. In contrast to the standard urn model, in our model we have additional control: namely, we can pull an arm, and thus govern the probability with which a new ball of the same color is added into the urn.

\subsection{The oracle and regret}\label{sec:oracle}

{\bf Maximizing expected reward.}  Throughout our presentation, we use $T$ to denote the time horizon over which performance is being optimized.  (The remainder of our paper characterizes upper and lower bounds on performance as the time horizon $T$ grows large.)  We let $\Gamma_T$ denote the total reward accrued up to time $T$:
\[
\Gamma _{T}=\sum_{t=1}^{T}X_{t}.
\]%
The goal of the platform is to choose a sequence $\{ I_t \}$ to maximize $\mathbb{E}[ \Gamma_T ]$.  As usual, $I_t$ must be a function only of the past history (i.e., prior to time $t$).

{\em {\bf The oracle policy.}}  As is usual in multiarmed bandit problems, we measure our performance against a benchmark policy that we refer to as the \Oracle{}.

\begin{definition}[\Oracle{}] 
The $\Oracle$ algorithm knows the optimal arm $a^*$, and pulls it at all times $t = 1, 2, \ldots$.
\end{definition}

Let $\Gamma_T^*$ denote the reward of the \Oracle{}.  Note that \Oracle{} may not be optimal for finite fixed $T$; in particular, unlike in the standard stochastic MAB problem, the expected cumulative reward $\mathbb E[\Gamma_T^*]$ is not $\mu_{a^*} T$, as several arrivals may not prefer the optimal arm. 

The next proposition provides tight bounds on $\mathbb E[\Gamma_T^*]$. For the proof, see the Appendix.

\begin{proposition}\label{prop:Oracle}
Suppose $\alpha>0$. Let $\theta^\alpha = \sum_{a\neq a^*} \theta_a^\alpha$.  The expected cumulative reward $\mathbb E[\Gamma_T^*]$ for the \Oracle{} satisfies:

1. $\displaystyle \mathbb E[\Gamma_T^*]  \le  \mu_{a^*}T - \mu_{a^*}\theta^\alpha\sum_{k=1}^T \frac{1}{(k + \theta_{a^*} - 1)^\alpha + \theta^\alpha}.$

2.  $\displaystyle \mathbb E[\Gamma_T^*] \ge  \mu_{a^*} T - \theta^\alpha \sum_{k=1}^{T} \frac{ 1}{(k+\theta_{a^*})^\alpha}-1 .$

In particular, we have:
\[  \mathbb E[\Gamma_T^*] = 
\begin{cases} 
 \mu_{a^*}T -  \Theta(T^{1-\alpha}), \quad &0<\alpha<1 \\
  \mu_{a^*}T -  \Theta(\ln T), \quad &\alpha = 1 \\
  \mu_{a^*}T - \Theta(1), \quad &\alpha > 1
\end{cases}
\] 
\end{proposition}

The discontinuity at $\alpha = 1$ in the asymptotic bound above arrises since $\sum_{k=1}^{T} \frac{1}{k^\alpha}$ diverges for each $\alpha \le 1$ but converges for $\alpha >1$. Further, the divergence is logarithmic for $\alpha =1$ but polynomial for each $\alpha <1$. 

Note that in all cases, the reward asymptotically is of order $\mu_{a^*} T$.  This is the best achievable performance to leading order in $T$, showing that the oracle is asymptotically optimal.

{\bf Our goal: Regret minimization.} Given any policy, define the {\em regret} against the \Oracle{} as $R_T$:
\begin{equation}
\label{eq:regret}
 R_T = \Gamma_T^* - \Gamma_T.
\end{equation}

Our goal in the remainder of the paper is to minimize the expected regret $\E[R_T]$. In particular, we focus on characterizing regret performance asymptotically to leading order in $T$ (both lower bounds and achievable performance), for different values of the externality exponent $\alpha$.


\section{Lower bounds}
\label{sec:lower}

In this section, we develop lower bounds on the achievable regret of any feasible policy.  As we will find, these lower bounds are quite distinct from the usual $O(\ln T)$ lower bound (see \cite{LaR85,Buc12}) on regret for the standard stochastic MAB problem.  This fundamentally different structure arises because of the positive externalities in the arrival process.

To understand our construction of the lower bound, consider the case where the externality function is linear ($\alpha = 1$); the other cases follow similar logic.
Our basic idea is that in order to determine the best arm, any optimal algorithm will need to explore all arms at least $\ln T$ times.  However, this means that after $t' = \Theta(\ln T)$ time, the total reward on any suboptimal arms will be of order $\sum_{b \neq a^*} N_b(t') = \Theta(\ln T)$.  Because of the effect of the positive externality, any algorithm will then need to ``recover'' from having accumulated rewards on these suboptimal arms.  We show that even if the optimal arm $a^*$ is pulled from time $t'$ onwards, a regret $\Omega(\ln^2 T)$ is incurred simply because arrivals who do not prefer arm $a^*$ continue to arrive in sufficient numbers. 

The next theorem provides regret lower bounds for all values of $\alpha$.  The proof can be found in the Appendix.

\begin{theorem}\label{thm:LowerBound}
\begin{enumerate}
\item For $\alpha < 1$, there exists no policy with $ \mathbb E[R_T] = o (T^{1-\alpha}\ln^{\alpha} T)$ on all sets of Bernoulli reward distributions.
\item For $\alpha = 1$, there exists no policy with $ \mathbb E[R_T] = o (\ln^2 T)$ on all sets of Bernoulli reward distributions.
\item For $\alpha > 1$, there exists no policy with $ \mathbb E[R_T] = o (\ln^\alpha T)$ on all sets of Bernoulli reward distributions.
\end{enumerate}
\end{theorem}

The remainder of the paper is devoted to studying regret performance of classic algorithms (such as UCB), and developing an algorithm that achieves the lower bounds above.

%
%
%
%




\section{Suboptimality of classical approaches} 
\label{sec:ucb}

We devote this section to developing structural insight into the model, by characterizing the performance of two classical approaches for the standard stochastic MAB problem: the UCB algorithm \cite{ACF02,Buc12} and a random-explore-then-commit algorithm.

\subsection{UCB}

We first show that the standard upper confidence bound (UCB) algorithm, which does not account for the positive externality, performs poorly.  (Recall that in the standard MAB setting, UCB achieves the asymptotically optimal $O(\ln T)$ regret bound \cite{LaR85,Buc12}.)

Formally, the UCB algorithm is defined as follows. 

\begin{definition}[UCB($\gamma$)]
Fix $\gamma > 0$. For each $a\in A$, let $\hat\mu_{a}(0) = 0$ and for each $t>0$ let $\hat\mu_{a}(t) \coloneqq \frac{S_{a}(t-1)}{T_{a}(t-1)}$, under convention that $\hat\mu_{a}(t) = 0$ if $T_{a}(t-1)=0$. For each $a\in A$ let $u_a(0) = 0$ and for each $t>0$ let 
$$u_a(t) \coloneqq  \hat\mu_{a}(t)+  \sqrt{\frac{\gamma \ln t}{T_a(t-1)}}. $$
Choose:
$$I_t \in \arg\max_{a\in A} u_a(t),$$
with ties broken uniformly at random. 
\end{definition}

Under our model, consider an event where $a^* \not\in J_t$ but $I_t = a^*$: i.e., $a^*$ is pulled but the arriving user did not prefer arm $a^*$.   Under UCB, such events are self-reinforcing, in that they not only lower the upper confidence bound for arm $a^*$, resulting in fewer future pulls of arm $a^*$, but they also reduce the preference of  {\em future users} towards arm $a^*$. 

It is perhaps not surprising, then, that UCB performs poorly.  However, the impact of this self-reinforcement under UCB is so severe that we obtain a striking result: there is a strictly positive probability that the optimal arm $a^*$ will {\em never} see a positive reward, as shown by the following theorem.  An immediate consequence of this result is that the regret of UCB is linear in the horizon length.   The proof can be found in the Appendix. 

\begin{theorem} \label{thm:UCB}
Suppose $\gamma > 0$. Suppose that $f(x)$ is $\Omega\left(\ln^{1+\epsilon}(x)\right)$ for some $\epsilon >0$. For UCB($\gamma$) algorithm, there exists an $\epsilon' > 0$ such that
$$\mathbb P\left(\lim_{T \to \infty} {S_{a^*}(T) =0} \right) \ge \epsilon'.$$ 
In particular, the regret of UCB($\gamma$) is $O(T)$.
\end{theorem}


\subsection{Random-explore-then-commit}

UCB fails because it does not explore sufficiently.  In this section, we show that more aggressive unstructured exploration is not sufficient to achieve optimal regret.  In particular, we consider a policy that chooses arms independently and uniformly at random for some period of time, and then commits to the empirical best arm for the rest of the time.  

\begin{definition}[REC($\tau$)]\label{def:random_tau}
Fix $\tau\in \mathbb Z_+$. For each $1\le t\le \tau$, choose $I_t$ uniformly at random from set $A$. Let $\hat a^* \in \arg\max_a S_a(\tau)$, with tie broken at random. For $\tau <t < T$, $I_t = a^*$. 
\end{definition} 

The following theorem provides performance bounds for the REC($\tau$) policy for our model. The proof of this result takes advantage of multitype continuous-time Markov branching processes \cite{AtK68,Jan04}; it is given in the Appendix.

\begin{proposition}\label{prop:RandomRegret}
Suppose that $\theta_a$ for each $a \in A$ is a positive integer. Let $b = \arg \max_{a\neq{a^*}} \mu_a$. The following statements hold for the REC($\tau$) policy for any $\tau$:

1. If $0<\alpha<1$ then we have 
$\mathbb E[R_T] = \Omega(T^{1-\alpha} \ln^{\frac{\alpha}{1-\alpha}} T).$ 

2. If $\alpha=1$ then we have 
$\mathbb E[R_T] = \Omega\left(T^{\frac{\mu_b}{\mu_b+\theta_{a^*}\mu_{a^*}}}\right).$
%

3. If $\alpha > 1$ then we have 
$\mathbb E[R_T] = \Omega(T) .$
\end{proposition}

Thus, for $\alpha \leq 1$, the REC($\tau$) policy may improve on the performance of UCB by delivering sublinear regret.  Nevertheless this regret scaling remains suboptimal for each $\alpha$.  In the next section, we demonstrate that carefully structured exploration can deliver an optimal regret scaling (matching the lower bounds in Theorem \ref{thm:LowerBound}).


\section{Optimal algorithms}
\label{sec:waterfill}

In this section, we present an algorithm that achieves the lower bounds presented in Theorem \ref{thm:LowerBound}.  The main idea of our algorithm is to structure exploration by {\em balancing} exploration across arms; this ensures that the algorithm is not left to ``correct'' a potentially insurmountable imbalance in population once the optimal arm has been identified.  

We first present a baseline algorithm called {\em Balanced Exploration} (BE) that nearly achieves the lower bound, but illustrates the key benefit of balancing; this algorithm has the advantage that it needs no knowledge of system parameters.  We then use a natural modification of this algorithm called {\em Balanced Exploration with Arm Elimination} (BE-AE) that achieves the lower bound in Theorem \ref{thm:LowerBound}, though it uses some knowledge of system parameters in doing so.

\subsection{Balanced exploration}

The BE policy is cautious during the exploration phase in the following sense: it pulls the arm with least accrued reward, to give it further opportunity to ramp up its score just in case its poor performance was bad luck. At the end of 
the exploration phase, it exploits the empirical best arm for the rest of the horizon.  

To define BE, we require an auxiliary sequence $w_k$, $k = 1, 2, \ldots$, used to set the exploration time.  The only requirement on this sequence is that $w_k \to \infty$ as $k \to \infty$; e.g., $w_k$ could be $\ln\ln k$ for each postive integer $k$.  The BE algorithm is defined as follows.  

\begin{definition}
{\bf Balanced-Exploration (BE) Algorithm:} Given $T$, let $n = w_T \ln T$.
\begin{enumerate}
\item {\em Exploration phase:} Explore until the (random) time $\tau_n = \min(t:  S_b(t) \ge n\ \forall\ b \in A)\wedge T$, i.e., explore until each arm has incurred at least $n$ rewards, while if any arm accrues less than $n$ rewards by time $T$, then $\tau_n = T$.  Formally, for $1\le t \le \tau_n$, pull arm $x(t) \in \arg\inf_{a \in A} S_a(t-1)$, with ties broken at random.
\item {\em Exploitation phase:} Let $\hat a^* \in \arg\inf_{a \in A} T_a(\tau_n)$, with tie broken at random. For $ \tau_n +1 \le t \le T$, pull the arm $\hat a^*$.
\end{enumerate}
\end{definition}

Note that this algorithm only uses prior knowledge of the time horizon $T$, but no other system parameters; in particular, we do not need information on the strength of the positive externality, captured by $\alpha$.  Our main result is the following.   The proof can be found in the Appendix.

\begin{theorem}
\label{thm:BEregret}
Suppose $w_k$, $k = 0, 1, 2, \ldots$, is any sequence such that $w_k \to \infty$ as $k \to \infty$.  Then the regret of the BE algorithm is as follows:

1. If $0< \alpha < 1$ then $\mathbb E[R_T] = O(w_T^\alpha T^{1-\alpha}\ln^\alpha T)$.

2. If $\alpha = 1$ then $\mathbb E[R_T]  = O(w_T \ln^2 T)$.

3. If $\alpha >1$ then $\mathbb E[R_T]  = O(w_T^\alpha \ln^\alpha T)$.
\end{theorem}

In particular, observe that if $w_k = \ln \ln k$, then we conclude $E[R_T] = \tilde{O}(T^{1-\alpha} \ln^\alpha T)$ (if $0 < \alpha < 1$); $E[R_T] = \tilde{O}(\ln^2 T)$ (if $\alpha = 1$); and $E[R_T] = \tilde{O}(\ln^\alpha T)$ (if $\alpha > 1$).  Recall that the notation $f(T) = \tilde{O}(g(T))$ implies there exists $k > 0$ such that $f(T) = O(g(T)\ln^k \!g(T))$.

\subsection{Balanced exploration with arm elimination}

The BE algorithm very nearly achieves the lower bounds in Theorem \ref{thm:LowerBound}.  The additional ``inflation'' (captured by the additional factor $w_T$) arises in order to ensure the algorithm achieves low regret despite not having information on system parameters.  

We now present an algorithm which eliminates the inflation in regret by intelligently eliminating arms that have poor performance during the exploration phase by using upper and lower confidence bounds.  
The algorithm assumes the knowledge of $T$, $m$, $\alpha$, and $\theta_a$ for each $a$ to the platform (though we discuss the assumption on the knowledge of $\theta_a$ further below).
With these informational assumptions, $\lambda_a(t)$ for each $t$ can be computed by the platform. Below, $\hat{\mu}_a(t)$ is an unbiased estimate of $\mu_a$ given observations till time $t$, while $u_a(t)$ and $l_a(t)$ are its upper and lower confidence bounds.

\begin{definition}
{\bf  Balanced Exploration with Arm Elimination (BE-AE) Algorithm:} Given $T$, $m$, and $\alpha$, as well as $\theta_a$ for each $a \in A$, for each time $t$ and each arm $a$ define:
$$\hat{\mu}_a(t) = (T_a(t))^{-1} \sum_{k=1}^{t} \frac{X_k}{\lambda_a(k)} \mathbbm{1}(I_k = a).$$
Further, let $c = \min_{a,b \in A} \frac{\theta_a}{m(1+ \theta_b)}$. Define $u_a(t) = \hat{\mu}_a(t) +  5\sqrt{ \frac{ \ln T}{cT_a(t)} },$ and  $l_a(t) = \hat{\mu}_a(t) -  5\sqrt{ \frac{ \ln T}{cT_a(t)} }.$

Let $A(t)$ be the set of active arms at time $t$. At time $t=1$ all arms are active, i.e., $A(1) =A$. At each time $t$ pull arm 
$$I_t \in \arg\inf_{a \in A(t)} S_a(t-1),$$
 with ties broken lexicographically.
Eliminate arm $a$ from the active set if there exists an active arm $b \in A(t)$ such that $u_a(t) < l_b(t)$. 
\end{definition}

 The following theorem shows that the BE-AE algorithm achieves optimal regret, i.e., it meets the lower bounds in Theorem \ref{thm:LowerBound}. The proof can be found in the Appendix.

\begin{theorem}\label{thm:BEAEregret}
For fixed $m$ and $\alpha$, the regret under the BE-AE algorithm satisfies the following: 

1. If $0< \alpha < 1$ then $\mathbb E[R_T] = O(T^{1-\alpha}\ln^\alpha T).$

2. If $\alpha = 1$ then $\mathbb E[R_T] = O(\ln^2 T).$

3. If $\alpha >1$ then $\mathbb E[R_T] = O(\ln^\alpha T).$
\end{theorem}

As noted above, our algorithm requires some knowledge of system parameters.  We briefly describe an approach that we conjecture delivers the same performance as BE-AE, but without knowledge of $\theta_a$ for $a \in A$.  Given a small $\epsilon>0$, first run the exploration phase of the BE algorithm for $n= \epsilon \ln T$ time without removing any arm.  For $t$ subsequent to the end of this exploration phase, i.e., once $\epsilon \ln T$ samples are obtained for each arm, we have $N_a(t) = \epsilon \ln T + \theta_a$. Thus, the effect of $\theta_a$ on $\lambda_a(t)$ becomes negligible, and one can approximate $\lambda_a(t)$ by letting $N_b(t) = S_b(t)$ for each arm $b$.  We then continue with the BE-AE algorithm as defined above (after completion of the exploration phase).  We conjecture the regret performance of this algorithm will match BE-AE as defined above.  Proving this result, and more generally removing dependence on $T$, $m$, and $\alpha$, remain interesting open directions.


\section{Simulations}\label{sec:simulations}

Below, we summarize our simulation setup and then describe our main findings.

\begin{figure}
\begin{subfigure}{.5\textwidth}
  \centering
  \includegraphics[width=\linewidth]{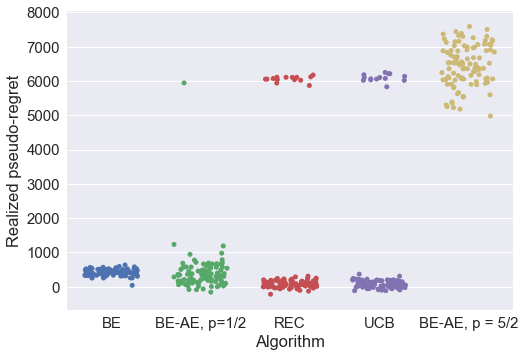}
  \caption{Realized psuedo-regret for $T=3\times 10^4$.}
  \label{fig:sfig1}
\end{subfigure}%
\begin{subfigure}{.5\textwidth}
  \centering
 \includegraphics[width=\linewidth]{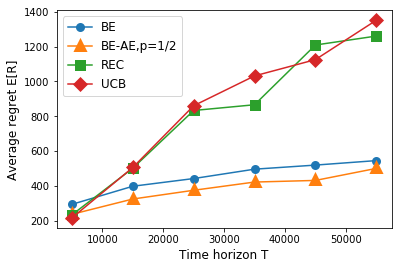}
  \caption{Expected regret as a function of time horizon $T$}
  \label{fig:sfig3}
\end{subfigure}
 \vspace{-1mm}
\caption{Performance comparison of algorithms in different parameter
regimes. All simulations have $m = 2$ arms, externality strength $\alpha =1$, arm reward parameters $\protect%
\mu_1 = 0.5$ and $\protect\mu_2 = 0.3$, and initial arm bias $\protect\theta%
_1 = \protect\theta_2 = 1$.}
\end{figure}

\textbf{Simulation setup.} We simulate our model with $m=2$ arms, with externality strength $\alpha =1$, arm
reward parameters $\mu _{1}=0.5$ and $\mu _{2}=0.3$, and initial biases $%
\theta _{1}=\theta _{2}=1$. For Fig.~\ref{fig:sfig1}, we
simulate each algorithm one hundred times for each set of parameters. We
plot \emph{pseudo-regret} realization from each simulation, i.e., $E[\Gamma
_{T}^{\ast }]-\Gamma _{T}$, where $E[\Gamma _{T}^{\ast }]$ is the expected
reward for the Oracle, computed via Monte Carlo simulation, and $\Gamma _{T}$
is the total reward achieved by the algorithm. Thus, lower pseudo-regret realization implies better performance. For Fig.~\ref{fig:sfig3}, each point is
obtained by simulating the corresponding algorithm one thousand times. The
time horizon $T$ is as mentioned in the
figures.

\emph{Parameters for each algorithm.} We simulate UCB($\gamma $) with $%
\gamma =3$. For Random-explore-then-commit, we set the exploration time as $%
\sqrt{T}$ (empirically, this performs significantly better than $\ln T$).
For BE, we set $w_{T}=\beta \ln \ln T$ with $\beta =2$ (see Definition 4).
For BE-AE, cf.~Definition 5, we recall that the upper and lower confidence
bounds are set as $u_{a}(t)=\hat{\mu}_{a}(t)+p\sqrt{\frac{\ln T}{T_{a}(t)}},$
and $l_{a}(t)=\hat{\mu}_{a}(t)-p\sqrt{\frac{\ln T}{T_{a}(t)}}$ for $%
p=5c^{-1/2}$ where $c=\min_{a,b\in A}\frac{\theta _{a}}{m(1+\theta _{b})}$.
This choice of $p$ was set in the paper for technical reasons, but unfortunately
this choice is suboptimal for finite $T$. The choice of $p=1/2$ achieves
significantly better performance for this experimental setup. The
performance is sensitive to small changes in $p$, as the plots illustrate
when choosing $p=5/2$. In contrast, in our experiments, we found that the
performance of BE is relatively robust to the choice of $\beta $.

\textbf{Main findings.} The following are our main findings from the above simulations.

First, even for $\alpha =1$, REC appears to
perform as poorly as UCB. Recall that in Section~\ref{sec:ucb} we show theoretically
that the regret is linear for UCB for each $\alpha$, and for REC for $\alpha >1$. For $\alpha =1$, we are only able to show that REC exhibits polynomial regret.

Second, for finite $T$, the performance of the (asymptotically optimal) BE-AE
algorithm is quite sensitive to the choice of algorithm parameters, and thus may
perform poorly in certain regimes. By contrast, the (nearly asymptotically
optimal) BE algorithm appears to exhibit more robust performance.

\section{Discussion and conclusions}\label{sec:discussion}

It is common that platforms make online decisions under uncertainty, and that these decisions impact future user arrivals. However, most MAB models in the past have decoupled the evolution of arrivals from the learning process. Our model, though stylized by design, provides several non-standard yet interesting insights which we believe are relevant to many platforms.  In particular:
\begin{enumerate}
\item In the presence of self-reinforcing preferences, there is a cost to being optimistic in the face of uncertainty, as mistakes are amplified.
\item It is possible to mitigate the impact of transients arising from positive externalities by structuring the exploration procedure carefully.
\item Once enough evidence is obtained regarding optimality of a strategy, one may even use the externalities to one's advantage by purposefully shifting the arrivals to a profit-maximizing population.
\end{enumerate}

Of course real-world scenarios are complex and involve other types of externalities which may reverse some of these gains. For example, the presence of negative externalities may preclude the ability to have ``all'' arrivals prefer the chosen option.  Alternatively, arrivals may have ``limited memory'', so that  future arrivals might eventually forget the effect of the externality. Overall, we believe that this is an interesting yet under-explored space of research, and that positive externalities of the kind we study may play a pivotal role in the effectiveness of learning algorithms.

\section{Acknowledgements}
This work was supported in part by National Science Foundation Grants CNS-1544548 and CNS-1343253.  Any opinions,
findings, and conclusions or recommendations expressed in this material are those of the authors and do not necessarily reflect the views of the National Science Foundation.

\bibliography{mybib.bib}
\bibliographystyle{plain}

\newpage


\section{Appendix}

\subsection{Proof of Proposition~\ref{prop:Oracle}}

We first show Part 1 of the result. Recall $\lambda_a(t)$ is the probability that the arrival at time $t$ prefers arm $a$, given the past. 
Thus, from the definition, we have
\begin{align*}
\mathbb E[\Gamma_T^*] & =  \mu_{a^*} \mathbb E\left[\sum_{t=1}^T\lambda_{a^*}(t)\right] \\
& = \mu_{a^*}\mathbb E\left[\sum_{t=1}^T \frac{N_{a^*}^\alpha(t-1)}{N_{a^*}^\alpha(t-1) + \theta^\alpha}\right] \\
& = \mu_{a^*}\mathbb E\left[\sum_{t=1}^T \left( 1- \frac{\theta^\alpha}{N_{a^*}^\alpha(t-1) + \theta^\alpha}\right)\right] \\
& = \mu_{a^*}T - \mu_{a^*}\mathbb E\left[\sum_{t=1}^T \frac{\theta^\alpha}{N_{a^*}^\alpha(t-1) + \theta^\alpha}\right].
\end{align*}

Using the fact that the maximum reward obtainable at each time is $1$, we obtain that $N_a^\alpha(t-1) \le \theta_a + (t-1)$. Thus,
%
 $$ \mathbb E[\Gamma^*_T] \le \mu_{a^*}T - \mu_{a^*}\sum_{t=1}^T \frac{\theta^\alpha}{(\theta_{a^*} + t -1)^\alpha + \theta^\alpha},$$ 
from which Part 1 follows. 

We now show Part 2 of the result. 
Let $\tau_1 = \inf (t: S_{a^*}(t) = 1)$,  i.e., it is the first time instant at which positive reward is obtained. For each $k>1$ let $\tau_k =  \inf (t: S_{a^*}(t) = k) - \tau_{k-1}$, i.e., it represents the time between $(k-1)\kth$ and $k\kth$ success. By definition,  
$\tau_k$ has distribution Geometric($\frac{\mu_{a^*}f(k+\theta_{a^*})}{f(k+\theta{a^*}) + \sum_{a\neq a^*} f(\theta_a)}$). One can view  $\Gamma_T^*$ as the minimum $n$ such that $ \sum_{k=1}^{n+1} \tau_k $ exceeds $T$. Thus, we have
$$T \le \mathbb E\left[ \sum_{k=1}^{\Gamma_T^*+1} \tau_k \right] .$$ 

Since $\tau_1,\tau_2,\ldots$ is a sequence of independent random variables, and since $\Gamma_T^*+1$ is a stopping time on this sequence, we obtain the following from Wald's lemma:

\begin{align*}
T & \le \mathbb  E\left[ \sum_{k=1}^{\Gamma^*_T+1}  \mathbb E[\tau_k] \right] \\
& \le \mathbb E\left[ \sum_{k=1}^{\Gamma^*_T+1} \frac{f(k+\theta_{a^*}) + \sum_{a\neq a^*} f(\theta_a)}{\mu_{a^*}f(k+\theta_{a^*})}  \right] \\
& = \mathbb E\left[ \sum_{k=1}^{\Gamma^*_T+1} \frac{(k+\theta_{a^*})^\alpha + \sum_{a\neq a^*} \theta_a^\alpha}{\mu_{a^*}(k+\theta_{a^*})^\alpha}\right]  \\
&= \mathbb E\left[ \sum_{k=1}^{\Gamma^*_T+1} \left(\frac{1}{\mu_a^*} + \frac{ \sum_{a\neq a^*} \theta_a^\alpha}{\mu_{a^*}(k+\theta_{a^*})^\alpha} \right) \right] 
\end{align*}

Thus we obtain,
\begin{align*}
T & \le  \frac{1}{\mu_{a^*}}  \mathbb E[\Gamma^*_T+1] + \mathbb E\left[  \sum_{k=1}^{\Gamma^*_T+1} \frac{\sum_{a\neq a^*} \theta_a^\alpha}{\mu_{a^*}(k+\theta_{a^*})^\alpha}  \right] \\
& \le  \frac{1}{\mu_{a^*}} \mathbb E[\Gamma^*_T+1] + \mathbb E\left[  \sum_{k=1}^{T} \frac{ \theta^\alpha}{\mu_{a^*}(k+\theta_{a^*})^\alpha}  \right]
\end{align*}

By rearranging we get,

$$ \mathbb E[\Gamma^*+1] \ge \mu_{a^*} T - \theta^\alpha \sum_{k=1}^{T} \frac{ 1}{(k+\theta_{a^*})^\alpha},$$
from which the result follows.

\subsection{Proof of Theorem~\ref{thm:LowerBound}}

We show the result for $\alpha = 1$ and $m=2$. For other values of $\alpha$ and $m$, the result follows in a similar fashion. 

Consider a problem instance where $A = \{a,b\}$, with expected rewards $\mu_a$ and $\mu_b$ respectively.  Without loss of generality, assume that $ \mu_b < \mu_a < 1$. 
The rewards obtained by a policy can be simulated as follows. Let $X_{1,a}, X_{2,a},\ldots$ be a sequence of i.i.d.\ Bernoulli($\mu_a$) random variables. Similarly, let $X_{1,b}, X_{2,b},\ldots$ be a sequence of i.i.d.\ Bernoulli($\mu_b$) random variables. Let $J_t$ represent the set of arms preferred by the arrival at time $t$. Recall that $I_t$ repesents the arm pulled at time $t$. Then the rewards obtained until time $t$, denoted $\Gamma_t$, are given by:

$$\Gamma_t = \sum_{k = 1}^t \Big( \mathbbm 1(I_k = a) \mathbbm 1(a \in J_k) X_{k,a} + \mathbbm 1(I_k = b)\mathbbm 1(b \in J_k) X_{k,b} \Big).
$$




First, we study the following \Oracle{}, and in particular characterize the maximum payoff achievable.  We then use this device to rule out the possibility of policies achieving the performance in the theorem statement.

\begin{definition}[\Oracle{$(t')$}] \label{def:Oracletprime}
Fix time $t'$. The values $\mu_a$, $\mu_b$ are revealed to the \Oracle{$(t')$} after time $t'$.
\end{definition}

\begin{lemma}\label{lemma:OracleLB}
Suppose $t' = o(T)$. Suppose the \Oracle{$(t')$} pulls arm $a$ at all times after $t'$. Then the total expected rewards obtained after time $t'$ by the \Oracle{} is $\mathbb E[\Gamma_T - \Gamma_{t'}] = \mu_a (T-t') - O(\mathbb E[N_b(t')] \ln T)$. 
\end{lemma}

{\em Proof of Lemma \ref{lemma:OracleLB}.}
The lemma is analogous to Part (ii) of Proposition~\ref{prop:Oracle}, with $\theta^\alpha$ replaced by $N_b(t')$, and measuring rewards at times greater than $t'$;  thus the lemma can be proved using arguments similar to those used in the theorem.\hfill$\Box$\\

The following lemma bounds the payoff achievable by the {\Oracle{}} after time $t'$.

\begin{lemma}\label{lemma:Oraclet'UB}
Suppose $t' = o(T)$.  Any policy used by the \Oracle{$(t')$} satisfies $\mathbb E[\Gamma_T - \Gamma_{t'}] = \mu_a (T-t') - \mathbb E[N_b(t')] \Omega(\mathbb \ln T)$.
\end{lemma}

{\em Proof of Lemma \ref{lemma:Oraclet'UB}.}  Consider any other policy for the \Oracle{}.  Let $\mathcal U_a(t')$ be the set of times at which arm $a$ is pulled after $t'$ and the arrival preferred arm $a$: $\mathcal U_a(t') = \{ t \geq t' : I_t = a, a \in J_t\}$.   Let $U_a(t') = |\mathcal{U}_a(t')|$.  It is clear that if $U_a(t')$ is $T - t' - \Omega(T)$, then the rewards obtained satisfy $E[\Gamma_T - \Gamma_{t'}] = \mu_a (T-t') - \Omega(T)$.  Thus, we assume without loss of generality that after time $t$', the \Oracle{} follows a policy with $U_a(t') = T - t' - o(T)$.

Using arguments similar to those used in Lemma~\ref{lemma:DelayedOracleUB}, we obtain:
$$\mathbb E[\Gamma_T - \Gamma_{t'}]  \le  \mu_{a} U_a(t')  - \sum_{t\in \mathcal U_a(t')}  \frac{\mathbb E[N_b(t')]}{t+(t'+\theta_b)} + \mu_b (T- t' - U_a(t')) .
$$
 
Since $\mu_b (T- t' - U_a(t')) \le \mu_a (T- t' - U_a(t')) $, we obtain:
\begin{align*}
\mathbb E[\Gamma_T - \Gamma_{t'}] &  \le  \mu_{a} (T-t')  - \sum_{t\in \mathcal U_a(t')}  \frac{\mathbb E[N_b(t')]}{t+(t'+\theta_b)} \\
&\le  \mu_{a} (T-t')  - \sum_{t=t'}^{U_a(t')}  \frac{\mathbb E[N_b(t')]}{t+(t'+\theta_b)} \\
& = \mu_a (T-t') - \mathbb E[N_b(t')] \Omega(\ln U_a(t')).
\end{align*}
Since $U_a(t') = O(T)$, the lemma follows. 

The preceding two lemmas establish that for any $t' = o(T)$, it is asymptotically optimal for the \Oracle{($t'$)} to always pull the best arm after time $t'$.  Since the \Oracle{$(t')$} has access to more information, it places a bound on the best achievable regret performance after {\em any} time $t'$.

Now suppose we are given any policy that has $\mathbb E[R_T] = O(\ln^2 T)$.  Consider time $t' = T^\gamma$, $\gamma > 0$.  For any time $t$ let $\mathcal T_a(t) = \{s \le t : I_s = a, a\in J_s\}$; these are the times prior to $t$ when arm $a$ was preferred by the arrival, and was subsequently pulled, and similarly define $\mathcal T_b(s) = \{ s \leq t : I_s = b, b \in J_s \}$. Further, define $\tilde T_a(t) = |\mathcal T_a(t)|$ and $\tilde T_b(t) = |\mathcal T_b(t)|$. 

Fix a constant $\mu_b'$ such that $\mu_a<\mu_b'<1$.  Consider the following three events, where $c_1= \frac{1}{2}\frac{\mu_b}{\mu'_b} \gamma$:
\begin{gather}
E_1 := \{ N_a(t') \le c_1 \ln T\}; \\ 
E_2: = \{ N_a(t') > c_1 \ln T,  N_b(t') > c_1 \ln T\};\\
E_3: = \{  N_a(t') > c_1 \ln T,  N_b(t') \le c_1 \ln T\}.\
\end{gather}

First, note that $R_{t'} = \Omega(1)$ since the Oracle as defined in Section~\ref{sec:oracle} is asymptotically optimal. Thus, it suffices to study $\mathbb{E}[R_T - R_{t'}]$. 

We trivially have:
$$\mathbb E[R_T - R_{t'}] = \mathbb E[(R_T - R_{t'}) \mathbbm 1(E_1)] + \mathbb E[(R_T - R_{t'}) \mathbbm 1(E_2)] + \mathbb E[(R_T - R_{t'}) \mathbbm 1(E_3)]. 
$$
We analyze each of these terms in turn.

  Under $E_1$, the total rewards obtained satisfy $\mathbb E[\Gamma_T - \Gamma_{t'}]  \le \mu_b O(T^\gamma) + \mu_a(T- T^\gamma)$. By our preceding analysis, the \Oracle{$(t')$} obtains reward $\mu_a T - \Theta(\ln T) $ in the same period. Since $\mu_a > \mu_b$, we have that $\mathbb E[R_T - R_{t'}|E_1] = \Omega(T^\gamma)$.  In particular, this implies that for any policy with $\mathbb E[R_T] = O(\ln^2 T)$, we must have $\mathbb P(E_1) = o(1)$.

Under $E_2$, we have $E[N_b(t')] \ge c_1\ln T$. From Lemma~\ref{lemma:Oraclet'UB} we have that $\mathbb E[R_T - R_{t'}|E_2] = \Omega(\ln^2 T)$.  

Thus, we have that 
$$ \mathbb E[R_T - R_{t'}] \ge \Omega(\ln^2T)\mathbb P(E_2) + \E[R_T - R_{t'}|E_3] \P(E_3),$$ 
where $\mathbb P(E_1) = o(1)$.  To conclude the proof, therefore, it suffices to show that $\mathbb P(E_3) = o(1)$ as well, since we have that $|\mathbb E[R_T - R_{t'}|E_2]| = O(\log^2 T)$ from Lemma \ref{lemma:Oraclet'UB}.

We prove this by considering a modified setting where the reward distribution for arm $a$ is Bernoulli($\mu_a$) (as in the original setting), and where the reward distribution for arm $b$ is Bernoulli($\mu'_b$). Recall,  $\mu_a< \mu'_b<1$. Thus, for the modified setting, arm $b$ is optimal. 

We let $\mathbb P$ ($\mathbb E$) and $\mathbb P'$ ($\mathbb E'$) denote the probability measure (resp., expectation) corresponding to the original and modified settings, respectively.  

It is elementary to show that:
$$\mathbb P'(E_3)  = \mathbb E[ \mathbbm 1(E_3) e^{-\hat K_{t'}(\mu_b,\mu_b')}]$$
%
%
%
where:
$$\hat{K}_t(\mu_b,\mu'_b)  =  \sum_{s \in \mathcal T_b(t)} \left(X_{s,b} \ln \frac{\mu_b}{\mu'_b} + (1-X_{s,b}) \ln \frac{1-\mu_b}{1-\mu'_b}\right).$$


Under the modified setting, again using our analysis of the \Oracle{$(t')$}, we know the regret incurred conditioned on $E_3$ is $\Omega(T^\gamma)$.   Thus for our candidate algorithm we have:
$$O(\ln^2 T) = E[R_T - R_{t'}] \ge \mathbb P'(E_3) \Omega(T^\gamma).$$ 

Thus we obtain $\mathbb P'(E_3) = O(T^{-\gamma} \ln^2 T)$.  Therefore, $\mathbb E[ \mathbbm 1(E_3) e^{-\hat K_{t'}(\mu_b,\mu_b')}] \le  O(T^{-\gamma} \ln^2 T)$. 

But under $E_3$ we have that $\hat K_{t'}(\mu_b,\mu_b') \ge c_1 \ln T \ln \frac{\mu_b}{\mu'_b}$, where the right hand side is the value obtained when $X_{t,b}$ for each $t \in \mathcal T_b(t')$ is $1$. Thus, we get
\begin{align}\label{eq:E3bound}
\mathbb P(E_3)  \le  e^{c_1 \ln T \ln \frac{\mu_b}{\mu'_b}} O(T^{-\gamma} \ln^2 T) 
		= O( T^{c_1 \frac{\mu'_b}{\mu_b} - \gamma} \ln^2 T).
\end{align}

But recall that $c_1 = \frac{1}{2}\frac{\mu_b}{\mu'_b} \gamma$. Thus we get $P(E_3) = o(1)$, and in turn, $E[R_T- R_{t'}] = \Omega(\ln^2 T)$, as required.

This completes the proof for $\alpha =1$. For $0<\alpha<1$, following along the lines of Lemma~\ref{lemma:Oraclet'UB}, we obtain that any policy used by the \Oracle{$(t')$} satisfies $\mathbb E[\Gamma_T - \Gamma_{t'}] = \mu_a (T-t') - \mathbb E[(N_b(t'))^\alpha] \Omega( T^{1-\alpha})$, and similarly for $0<\alpha<1$ we have $E[\Gamma_T - \Gamma_{t'}] = \mu_a (T-t') - \mathbb E[N_b(t')^\alpha] \Omega(1)$. Further, for $0<\alpha<1$, we set $\gamma > 1-\alpha$, $c_1 = \frac{1}{2}\frac{\mu_b}{\mu'_b}( \gamma - 1+\alpha)$ so that bound equivalent to \eqref{eq:E3bound} on $\mathbb P(E_3)$ for this case is $o(1)$. Rest of the proof follows from arguments similar to that $\alpha = 1$.

\subsection{Proof of Theorem~\ref{thm:UCB}}

We first prove the result for the setting with two arms, i.e., $m=2$, and then generalize later. Suppose $A = \{a,b\}$. Without loss of generality, let $\mu_a > \mu_b$. 

Let $\tau_k$ be the time at which arm $a$ is pulled for the $k\kth$ time. 

Let $Q_k$ be the event that the first $k$ pulls of arm $a$ each saw a user which did not prefer arm $a$. 

Let $E_k$ be the event that $\hat \mu_b(\tau_k - 1) > \frac{\theta_b \mu_b}{3}$. 

Then, under $Q_{k-1} \cap E_{k-1}$, we have the following for each time $t$ s.t.\ $\tau_{k-1}< t\le e^{\left(\frac{\theta_b\mu_{b}}{4 }\right)^2\frac{k-1}{\gamma}}$:

$$u_a(t) < \sqrt\frac{ \gamma \ln e^{\left(\frac{\theta_b\mu_{b}}{4 }\right)^2\frac{k-1}{\gamma}} }{ k-1} = \frac{\theta_b\mu_b}{4} < \frac{\theta_b\mu_b}{3} < \hat \mu_b(t) < u_b(t).$$

Thus, under $Q_{k-1} \cap E_{k-1}$, arm $b$ is pulled for each time $t$ s.t.\ $\tau_{k-1}< t\le e^{\left(\frac{\theta_b\mu_{b}}{4 }\right)^2\frac{k-1}{\gamma}}$, which in turn implies that $\tau_k \ge e^{\left(\frac{\theta_b\mu_{b}}{4 }\right)^2\frac{k-1}{\gamma}}$. 

We now show that there exists an $\epsilon'> 0$ such that $ \liminf \limits_{k\to \infty} \mathbb P(Q_k \cap E_k) \ge \epsilon'$ from which the result would follow. 

Using law of total probability we have,
$$ \mathbb P(Q_k\cap E_k) \ge \mathbb P(Q_{k-1} \cap E_{k-1})\mathbb P(Q_k\cap E_k | Q_{k-1}, E_{k-1}).$$

Thus, we have
\begin{equation}\label{eqn:QkEk_Expansion}
\mathbb P(Q_k\cap E_k) \ge \mathbb P(Q_{k-1} \cap E_{k-1}) \mathbb P(E_k |  Q_{k-1}, E_{k-1})  \mathbb P(Q_k|  Q_{k-1}, E_{k-1}, E_k).
\end{equation}

Note that, under $Q_{k-1}$, arm $b$ is pulled at least $k-1$ times before $\tau_k$. Using standard Chernoff bound techniques it is easy to show that there exists a constant $\delta'$ such that
$\mathbb P(E_k, E_{k-1} |  Q_{k-1}) \ge 1 - e^{-\delta'(k-1)}$.
(This can be shown using the standard approach for deriving Chernoff bounds, but with the following version of Markov inequality: $P(X>a,Y>b) \le E[XY]/(ab)$.) Thus, we get
\begin{equation}\label{eqn:EkConcentration}
\mathbb P(E_k |  Q_{k-1}, E_{k-1}) \ge \mathbb P(E_k, E_{k-1} |  Q_{k-1})  \ge 1 - e^{-\delta'(k-1)}.
\end{equation}

Under $Q_{k-1} \cap E_{k-1} \cap E_k$, we have that $N_a(\tau_k -1) = \theta_a$ and 
$$N_b(\tau_k-1) = \theta_b + S_b(\tau_k-1) = \theta_b + \hat \mu_b (\tau_k -1) T_b(\tau_k - 1).$$ Further, since $\tau_{k} \ge e^{\left(\frac{\theta_b\mu_{b}}{4 }\right)^2\frac{k-1}{\gamma}}$, we have
$$T_b(\tau_k - 1) \ge \max\left(k-1, e^{\left(\frac{\theta_b\mu_{b}}{4 }\right)^2\frac{k-1}{\gamma}} - k + 1\right).$$
 Thus, we have
 $$N_b(\tau_k-1) \ge  \theta_b + \frac{\theta_b \mu_b}{3} \max\left(k-1, e^{\left(\frac{\theta_b\mu_{b}}{4 }\right)^2\frac{k-1}{\gamma}} - k + 1\right).$$
Thus there exists a constant $c>0$ such that the following holds for each $k\ge 2$: under $Q_{k-1} \cap E_{k-1} \cap E_k$ we have that 
$$ N_b(\tau_k-1) \ge e^{c(k-1)}.$$

Thus, under $Q_{k-1} \cap E_{k-1} \cap E_k$, we have 
$$\lambda_a(\tau_k-1) = \frac{\theta_a}{f\left(N_b(\tau_k-1)\right) + \theta_a} \le  \frac{\theta_a}{f\left(e^{c(k-1)}\right) + \theta_a}.$$ 

Thus, from definition of $Q_k$ we have 
\begin{equation}\label{eqn:QkBound}
 \mathbb P(Q_k|  Q_{k-1}, E_{k-1}, E_k) \ge 1 -  \frac{\theta_a}{f\left(e^{c(k-1)}\right) + \theta_a} =  \frac{f\left( e^{c(k-1)} \right)}{f(\theta_a) + f\left( e^{c(k-1)} \right)}.
 \end{equation}

Substituting \eqref{eqn:QkBound} and \eqref{eqn:EkConcentration} in \eqref{eqn:QkEk_Expansion}, we obtain

$$ \mathbb P(Q_k\cap E_k) \ge \mathbb P(Q_{k-1} \cap E_{k-1}) \left(1 - e^{-\delta'(k-1)}\right)  \left( \frac{f\left( e^{c(k-1)} \right)}{f(\theta_a) + f\left( e^{c(k-1)} \right)}\right).
$$ 
Computing recursively, we obtain
 
$$ \mathbb P( Q_k  \cap E_k) \ge  \mathbb P(Q_{2} \cap E_{2})  \prod_{l = 2}^k \left(1 - e^{-\delta'(k-1)}\right) 
   \prod_{l = 2}^k  \left( \frac{f\left( e^{c(l-1)} \right)}{f(\theta_a) + f\left( e^{c(l-1)} \right)}\right).
$$
Thus, we would be done if we show that  $ \lim\inf_{k \to \infty} \sum_{l = 2}^k \ln \left(1 - e^{-\delta'(k-1)}\right)$ as well as that  $ \lim\inf_{k \to \infty} \sum_{l = 2}^k \ln  \left( \frac{f\left( e^{c(l-1)} \right)}{f(\theta_a) + f\left( e^{c(l-1)} \right)}\right)$ are both greater than $-\infty$. We show this below. We use the fact that $\ln(1-x) \ge -x$ for each $x>0$.  We have, 

$$\sum_{l = 2}^k \ln \left(1 - e^{-\delta'(k-1)}\right) \ge - \sum_{l = 2}^k  e^{-\delta'(k-1)}, $$
which tends to a constant a $k \to \infty$. 

Further, 

$$ \sum_{l = 2}^k \ln  \left( \frac{f\left( e^{c(l-1)} \right)}{f(\theta_a) + f\left( e^{c(l-1)} \right)}\right) \ge - \sum_{l = 2}^k   \left( \frac{f\left( \theta_a \right)}{f(\theta_a) + f\left( e^{c(l-1)} \right)}\right) $$ 
which tends to a constant a $k \to \infty$ since  $f(x)$ is $\Omega\left(\ln^{1+\epsilon}(x)\right)$. This completes the proof for $m = 2$. 

For $m > 2$, we can generalize the argument to show that only the worst arm will see non-zero rewards with positive probability by appropriately generalizing the notions of $\tau_k,E_k,$ and $Q_k$ and arguing along the above lines.

\subsection{Proof of Proposition~\ref{prop:RandomRegret}}
We start with a technical result for the algorithm that {\em indefinitely} pulls arms independently and uniformly at random.  For the case where $f$ is linear, we can model the cumulative rewards obtained at each arm via the {\em generalized Friedman's urn process}.  These processes  are studied by embedding them into multitype continuous-time Markov branching processes \cite{AtK68,Jan04}, where the expected lifetime of each particle is one at all times. 

Here, since we are interested in rewards obtained for more general $f$, we study this by considering multitype branching processes with state-dependent expected lifetimes. For technical reasons, we will assume that $\theta_a$ for each arm $a$ is integer valued and greater than or equal to $1$. This allows us to map our problem into an urn type process with initial number of balls of color $a$ in the urn being equal to $\theta_a$. We obtain the following result. 

\begin{proposition}\label{prop:RandomAsymptotics}
Suppose that $\theta_a$ for each $a \in A$ is a positive integer. Suppose at each time step $t$ an arm is pulled independently and uniformly at random.
The following statements hold:
\begin{enumerate}[(i)]
\item If $f(x) = x^\alpha$ for $0<\alpha<1$ then for each $b \neq a^*$, we have that $\frac{N_{a^*}(t)}{N_b(t)} \to\frac{\theta_a}{\theta_b} \left(\frac{\mu_{a^*}}{\mu_{b}}\right)^{\frac{1}{1-\alpha}}$ almost surely as $t \to  \infty$. 
\item If $f(x) = x$ then for each $b \neq a^*$, we have that $\frac{N_{a^*}(t)}{(N_{b}(t))^\frac{\mu_{a^*}} {\mu_b}}$ converges  almost surely to a random variable $Y$ with $0<Y<\infty$ w.p. $1$. 
\item If $f(x) = x^\alpha$ for $\alpha > 1$ then there is a positive probability that $N_{a^*}(t)$ is $O(1)$ while for some $b \neq a^*$ we have $N_b(t) \to \infty $ as $t \to \infty$.
\end{enumerate}
\end{proposition}

\begin{proof}

For ease of exposition we will assume that $A = \{a,b\}$. The argument for the more general case is more or less identical. 

For now, suppose that $\theta_a = \theta_b = 1$. We will study the process $N = \left(N_a(t),N_b(t)\right)_{t\in \mathbb Z_+}$ by analyzing a multitype continuous time Markov branching process $Z = \left(Z_a(s), Z_b(s)\right)_{s \in \mathbb R_+}$ such that its embedded Markov chain, i.e., the discrete time Markov chain corresponding to the state of the branching process at its jump times, is statistically identical to $N(t)$. By jump time we mean the times at which a particle dies; upon death it may give birth to just one new particle, in which case, the size of the process may not change at the jump times. 

We construct $Z$ as follows. Both $Z_a$ and $Z_b$ are themselves independently evolving single dimensional branching processes. Initially, $Z_a$ and  $Z_b$ have one particle each, i.e., $|Z_a(0)| = |Z_b(0)| = 1$. Each particle dies at a rate dependent on the size of the corresponding branching processes as follows: at time $s$ each particle of $Z_a$ dies at rate $\frac{f(|Z_a(s)|)}{|Z_a(s)|}$. At the end of its lifetime, the particle belonging to $Z_a$ dies and gives birth to one new particle with probability $\frac{1 - \mu_a}{2}$ and two new particles with probability $\frac{\mu_a}{2}$. Similarly for the particles belonging to $Z_b$. 

We will use notation $|Z|$ to denote $(|Z_a(s_t)|,|Z_a(s_t)|)$.  We now show that the embedded Markov chain of $|Z|$ is statistically identical to $N$. 
 Let $s_1, s_2, \ldots , s_t, \ldots$ represent the jump times of $Z$. We show that the conditional distribution of $N(t)$ given $N(t-1)$ is identical to the conditional distribution of $|Z(s_t)|$ given $|Z(s_{t-1})|$. Since at each time $t$ an arm is chosen at random, we have 
$$\P\left((N_a(t), N_b(t)) = (N_a(t-1) +1 , N_b(t)) \big| N(t-1) \right) = \frac{1}{2} \frac{f(N_a(t-1))\mu_a}{f(N_a(t-1)) + f(N_b(t-1))} .$$
Similarly, we can compute the conditional probability for the other values which $N(t)$ can take. 
Now consider process $Z(\tau)$. After the $(t-1)\kth$ jump of $Z$, the rate at which $Z_a$ jumps is $f(|Z_a(s_t)|)$. Thus, the probability that the $(t+1)\kth$ jump of $Z$ belongs to $Z_a$ is $\frac{f(|Z_a(s_t)|)}{f(|Z_a(s_t)|) + f(|Z_b(s_t)|)}$. Further, each jump at $Z_a$ results into an increment with probability $\frac{\mu_a}{2}$.  Thus we have, 
$$\P\left((|Z_a(s_t)|,|Z_a(s_t)|) = (|Z_a(s_{t-1})+1|,|Z_a(s_{t-1})|) \big| Z(t-1) \right) = \frac{\mu_a}{2} \frac{f(|Z_a(s_t)|)}{f(|Z_a(s_t)|) + f(|Z_b(s_t)|)}.$$
Further, it is easy to check that $|Z(s_1)|$ and $N(1)$ are identically distributed. Thus, by induction, the embedded Markov chain of $|Z|$ is statistically identical to $N$.

Now, we obtain the following lemma from Theorem 1 in \cite{Kus83}. We say that $f$ is sublinear if there exists $0< \beta < 1$ such that $f(x) \le x^\beta$. 
\begin{lemma}
If $f(x)$ is linear or sublinear, then 
$$|Z_a(s)| \to w_a (s) (W + o(1)),$$ where $w_a(s)$ is the inverse function of 
$$g_a(s) = \frac{2}{\mu_a} \int_0^s \frac{1}{f(x)} dx,$$ and $W$ is a random variable with $0<W<\infty$ w.p.1. Moreover, $W = 1$ is $f$ is sublinear.    
\end{lemma}

Now, consider $f(x) = x^\alpha$ for $0<\alpha<1$. Then, it follows that $w_a(s) = \left(\frac{\mu_a}{s(1-\alpha)}\right)^{\frac{1}{1-\alpha}}$. 
Thus, we have
$$|Z_a(s)|  \left(\frac{2s(1-\alpha)}{\mu_a}\right)^{\frac{1}{1-\alpha}} \to 1 \quad \text{a.s.},$$
and
$$|Z_b(s)|  \left(\frac{2s(1-\alpha)}{\mu_b}\right)^{\frac{1}{1-\alpha}} \to 1 \quad \text{a.s.}.$$

Thus, part $(i)$ of the theorem follows for the case where $\theta_a = \theta_b = 1$. For general $\theta_a$ and $\theta_b$, we construct as many independent branching processes, apply the above lemma, and the result follows. 

Part $(ii)$ follows in a similar fashion and noting that $w_a (s) = e^{\frac{\mu_a}{2}s}$. 

We now argue for part $(iii)$. We assume that $\theta_a = \theta_b =1$, the argument for general $\theta_a$ and $\theta_b$ is similar. We show that if $f(x) = x^\alpha$ for $\alpha > 1$ then there exists a time $s <\infty$ such that $\mathbb P(|Z_b(s)| = \infty) > 0$. Our result follows from this since for each finite $s$ we have that $\mathbb P(Z_a(s) = 1) \ge e^{-s} > 0$. For each $k \ge 1$ let $\gamma_k = \inf \{s \in \mathbb R_+ : |Z_b(s)| = k\}$. Clearly, $\gamma_k - \gamma_{k-1}$ is the sum of a random number (with distribution Geometric$(\frac{2}{\mu_b})$) of Exponential$(f(k-1))$ distributed random variables. Thus,  $\mathbb E[\gamma_k] = \frac{2}{\mu_a} \sum_{l=1}^{k-1} \frac{1}{l^\alpha}$, which tends to a constant, say $\delta'$, as $k \to \infty$. Thus, $\mathbb P(|Z_b(\delta')| = \infty) > 0$. Hence part $(iii)$ follows. This completes the proof of Proposition~\ref{prop:RandomAsymptotics}.
\end{proof}

We now continue with proof of Proposition~\ref{prop:RandomRegret}. Recall the Definition~\ref{def:random_tau} for Random($\tau$) policy. We assume $\tau = o(T)$, since if not, $E[R_T]$ is $O(T)$ as arms are picked at random during exploration phase. 

Part $(iii)$ thus follows from Proposition~\ref{prop:RandomAsymptotics} and noting that $\mathbb P(\hat a^* \neq a^*)$ is $\Omega(1)$ while the exploitation phase runs for $T-\tau = O(T)$ time. 

We now show Part $(ii)$. We first show the following lemma. 

\begin{lemma}\label{lemma:alpha1wrongdetection}
For $\alpha = 1$, under Random($\tau$) policy we have $\mathbb P(\hat a^* \neq a^*)  = \Omega( \tau^{-\frac{\theta_{a^*}\mu_{a^*}}{\mu_b}})$. 
\end{lemma}
To prove the lemma, for now suppose that $\theta_a = 1$ for each arm $a$. Recall the continuous time Markov-chain branching process construction in the proof of Proposition~\ref{prop:RandomAsymptotics}. It is easy to generalize the construction for $m\ge 2$. For general $m$, in process $Z_a(s)$ for each arm $a$ the probability that upon death of a particle it gives birth to two new particles is $\frac{\mu_a}{m}$. For $\alpha=1$ the process $Z_a(s)$ is a equivalent to the well-known Yule Process~\cite{Yul25} and $|Z_a(s)|$ has distribution Geometric($e^{-s\mu_a/m}$) for each $s$. Thus, for each positive real $s$ and positive integer $k$ we have
$$\mathbb P(|Z_a(s)| > k)= (1- e^{-s\mu_a/m})^k.$$

Using $k = \tau$ and $s= \frac{m\ln\tau}{\mu_a}$ we obtain,
\begin{align*}
\mathbb P(|Z_a(s)| > \tau) = (1- e^{-\ln \tau})^\tau 
					 = (1-\frac{1}{\tau})^\tau 
\end{align*}

Now, let $s' = \sup(s: Z_{a^*}(s) =0)$. Clearly, $s'$ has Exponential($\frac{\mu_{a^*}}{m}$) distribution. Thus, for arm $b$, we have 
$$\mathbb P\left(s' >  \frac{m\ln\tau}{\mu_b}\right) = e^{-\frac{\mu_{a^*}\ln\tau}{\mu_b}} = \tau^{-\frac{\mu_{a^*}}{\mu_b}}.$$

Now, note that the event $\{s' >  \frac{m\ln\tau}{\mu_b}\} \cap \{|Z_b(s)| > \tau\}$ is a subset of the event 
$S_{a^*}(\tau) = 0$. Thus, 
\begin{multline*}
\mathbb P(\hat a^* \neq a^*) \ge \mathbb P(s' >  \frac{m\ln\tau}{\mu_b},|Z_b(s)| > \tau) = (1-\frac{1}{\tau})^\tau  \tau^{-\frac{\mu_{a^*}}{\mu_b}} = \Omega( \tau^{-\frac{\mu_{a^*}}{\mu_b}}).
\end{multline*}
Hence, the lemma follows for the case where $ \theta_a =1$ for each arm $a$. For the general values of $\theta_a$, note that we only get an upper bound on $\mathbb P(\hat a^* \neq a^*)$ if we assume that $\theta_a =1$ for each $a \neq a^*$. Hence, we assume that $\theta_a = 1$ for each $a \neq a^*$. Then, the lemma follows by the same arguments as above and nothing that $s'$ now has Exponential($\frac{\theta_{a^*}\mu_{a^*}}{m}$) distribution.  


We now consider two cases seperately: Case 1 consists of $\tau \le T^{\frac{\mu_b}{\mu_b+\theta_{a^*}\mu_{a^*}}}$, and Case 2 consists of $\tau \ge T^{\frac{\mu_b}{\mu_b+\theta_{a^*}\mu_{a^*}}}$. 

\underline{Case 1 ($\tau \le T^{\frac{\mu_b}{\mu_b+\theta_{a^*}\mu_{a^*}}}$):} By Law of Total Expectation, we have 
$$\mathbb E[R_T]  \ge   \mathbb E[R_T|\hat a \neq \hat a^*] \mathbb P(\hat a \neq \hat a^*).$$
Since $\tau = o(T) $ we have that $\mathbb E[R_T|\hat a \neq \hat a^*]  = O(T)$. Thus, 
$$ \mathbb E[R_T]  = \Omega(T)  \mathbb P(\hat a \neq \hat a^*) = \Omega(T  \tau^{-\frac{\mu_{a^*}}{\mu_b}}),$$
where the last inequality follows from Lemma~\ref{lemma:alpha1wrongdetection}. Since $\tau \le T^{\frac{\mu_b}{\mu_b+\theta_{a^*}\mu_{a^*}}}$, we have 
$$ \mathbb E[R_T] \ge  \Omega(T \times  T^{-\frac{\mu_{a^*}}{\mu_b+\theta_{a^*}\mu_{a^*}}}),$$
from which the result follows. 

\underline{Case 2 ($\tau > T^{\frac{\mu_b}{\mu_b+\theta_{a^*}\mu_{a^*}}}$):} Clearly, regret is $\Omega(\tau)$. Thus, we again get $E[R_T]  = \Omega(T^{\frac{\mu_b}{\mu_b+\theta_{a^*}\mu_{a^*}}}),$
from which the result follows. 

This completes the proof of Part $(ii)$. 

We now show Part $(i)$. Here again we bound $\mathbb P(\hat a^* \neq a^*) $ from below by $\mathbb P(S_{a^*}(\tau) = 0)$, but we use a more direct approach than considering continuous time branching processes. 

\begin{lemma}\label{lemma:alphaLess1wrongdetection}
For $0<\alpha <1$, there exists a constant $c$ such that under Random($\tau$) policy we have $\mathbb P(\hat a^* \neq a^*)  \ge  e^{-  c\left(\tau^{1-\alpha}\right)})$. 
\end{lemma}
Consider an experiment where each arm is pulled at random at each time $t=1,2, \ldots, \infty.$ Let $\tau_1 , \tau_2,\ldots\infty.$ be the times at which the reward obtained is $1$ while the arm being pulled is either arm $a^*$ or arm $b$. Since arms are pulled at random, we have
$$\mathbb P(I_{\tau_1} = b) = \frac{\theta_b^\alpha}{\theta_b^\alpha + \theta_{a^*}^\alpha}.$$
Note that this probability does not depend on the $\theta_a$ for $a\notin \{a^*,b\}$. Similarly, for each $k\ge1$,
$$\mathbb P\left(I_{\tau_{k+1}} = b \bigg|\bigcap_{l=1}^k I_{\tau_l} = b\right) = \frac{(\theta_b + k)^\alpha}{(\theta_b + k)^\alpha + \theta_{a^*}^\alpha}.$$

Thus,
\begin{align*}
\mathbb P\left(\bigcap_{k=1}^\tau I_{\tau_k} = b\right) 
	& = \prod_{k=1}^\tau \mathbb P\left(I_{\tau_{k}} = b \bigg|\bigcap_{l=1}^{k-1} I_{\tau_l} = b\right) \\
	& =   \prod_{k=1}^\tau \frac{(\theta_b + k-1)^\alpha}{(\theta_b + k-1)^\alpha + \theta_{a^*}^\alpha} \\
	& = \prod_{k=1}^\tau e^{ - \ln \frac{(\theta_b + k-1)^\alpha + \theta_{a^*}^\alpha}{(\theta_b + k-1)^\alpha}} \\
	& =  e^{ - \sum_{k=1}^\tau \ln   \frac{(\theta_b + k-1)^\alpha + \theta_{a^*}^\alpha}{(\theta_b + k-1)^\alpha}} \\
	& =  e^{ - \sum_{k=1}^\tau \ln (1+  \frac{ \theta_{a^*}^\alpha}{(\theta_b + k-1)^\alpha})} \\
	& \ge  e^{ - \sum_{k=1}^\tau  \frac{ \theta_{a^*}^\alpha}{(\theta_b + k-1)^\alpha}} \\
	& \ge e^{ - \sum_{k=1}^\tau  \frac{ \theta_{a^*}^\alpha}{(1+ k-1)^\alpha}} \\
	& \ge e^{- \Theta \left(\tau^{1-\alpha}\right)}.
\end{align*}

Under Random($\tau$) policy, the maximum number of successes possible by either arm $a^*$ or $b$ in the exploration phase is $\tau$.  Thus, $\mathbb P(\bigcap_{k=1}^\tau I_{\tau_k} = b) $ as computed above is a lower bound on $ P(\hat a^* \neq a^*)$. This complete the proof of the lemma.

Similar to $\alpha = 1$, here again we consider two cases: Case 1 consists of $\tau \le  \frac{c}{2\alpha}\ln^{\frac{1}{1-\alpha}T} $, and Case 2 consists of $\tau \ge  \frac{c}{2\alpha}  \ln^{\frac{\alpha}{1-\alpha}T} $, where $c$ is the constant from Lemma~\ref{lemma:alphaLess1wrongdetection}.

\underline{Case 1 ($\tau <  \frac{c}{2\alpha} \ln^{\frac{1}{1-\alpha}}T $):} Using argument similar to that for $\alpha = 1$, we have 
$$ \mathbb E[R_T]   \ge \Omega(T)  \mathbb P(\hat a \neq \hat a^*) = \Omega( T e^{- c \tau^{1-\alpha}})
 = \Omega (T e^{-\frac{\alpha}{2} \ln T})  = \Omega(T^{1-\frac{\alpha}{2}}) =\Omega(T^{1-\alpha} \ln^{\frac{\alpha}{1-\alpha}} T),$$
from which the result follows.
 
\underline{Case 2 ($\tau \ge  \frac{c}{2\alpha} \ln^{\frac{1}{1-\alpha}} T$):} From Part $(i)$ of Proposition~\ref{prop:RandomAsymptotics}, as $\tau\to \infty$ we have that
$\frac{N_{a}^\alpha(\tau)}{N_{a'}^\alpha(\tau)} $ tends to a constant for each pair of arms $a, a'$. Further, $\sum_a N_{a}(\tau) \le \tau$. Thus, we have 
$\mathbb E[N_a^\alpha(\tau)] = \Omega(\tau^\alpha)$ for each arm $a$. 
 In other words, there exists a positive constants, say $\beta$, such that $\mathbb E[ N_a(\tau)] \ge \beta \tau^\alpha$ for each $\tau$. 

Now consider the exploitation phase. Let $\Gamma'$ be the rewards accrued during this phase. We provide below a bound on $\mathbb E[\Gamma']$. 
\begin{lemma}\label{lemma:DelayedOracleUB}
The rewards accrued during exploitation phase satisfies:
$$\mathbb E[\Gamma'] \le \mu_{a^*} (T - \tau) - \sum_{t=\tau+1}^T   \frac{\beta \tau^\alpha}{t^\alpha+(\tau+\theta_b)^\alpha} .$$
\end{lemma}
The lemma can be shown as follows. 
\begin{align*}
\mathbb E[\Gamma'] &\le \mu_{a^*}  \mathbb E [\sum_{t=\tau+1}^T \lambda_a(t)] \\
		&  = \mu_{a^*} \mathbb E[\sum_{t=\tau+1}^T  \frac{N_{a^*}^\alpha(t-1)}{\sum_a N_{a^*}^\alpha(t-1)} ] \\
		&  \le \mu_{a^*} \mathbb E[\sum_{t=\tau+1}^T \frac{N_{a^*}^\alpha(t-1)}{ N_{a^*}^\alpha(t-1) +N_b^\alpha(t-1)} ] \\
		&  =  \mu_{a^*} \mathbb E[\sum_{t=\tau+1}^T \left(1-   \frac{N_{b}^\alpha(t-1)}{N_{a^*}^\alpha(t-1) +N_b^\alpha(t-1)} \right) ] \\
		& = \mu_{a^*} (T-\tau) - \sum_{t=\tau+1}^T \mathbb E[ \frac{N_{b}^\alpha(t-1)}{N_{a^*}^\alpha(t-1) +N_b^\alpha(t-1)} ] \\
		& \le \mu_{a^*} (T-\tau) - \sum_{t=\tau+1}^T \mathbb E[ \frac{N_{b}^\alpha(\tau)}{N_{a^*}(t-1) +N_b^\alpha(\tau)} ] \\
		& \le \mu_{a^*} (T - \tau) - \sum_{t=\tau+1}^T \mathbb E[ \frac{N_{b}^\alpha(\tau)}{t^\alpha+(\tau+\theta_b)^\alpha} ] \\
		& \le \mu_{a^*} (T - \tau) - \sum_{t=\tau+1}^T   \frac{\beta \tau^\alpha}{t^\alpha+(\tau+\theta_b)^\alpha} 
\end{align*}
Hence the lemma follows. Further, the maximum rewards accrued during exploration phase if $\mu_{a^*}\tau$. Thus, the overall expected rewards $\E[\Gamma]$ satisfies 
$$\E[\Gamma] \le \mu_{a^*}\tau + \E[\Gamma'] \le \mu_{a^*} T - \sum_{t=\tau+1}^T    \frac{\beta \tau^\alpha}{t^\alpha+(\tau+\theta_b)^\alpha}  .$$ 
Thus, from above inequality and from Proposition~\ref{prop:Oracle} we have

$$\E[R_T] = \E[\Gamma^*] - \E[\Gamma] \ge  - \Theta(T^{1-\alpha}) + \beta \tau^\alpha \sum_{t=\tau+1}^T   \frac{1}{t^\alpha+(\tau+\theta_b)^\alpha} $$ 

Thus, we have
%
\begin{align*}
\E[R_T] & \ge 
        - \Theta(T^{1-\alpha}) + \beta \tau^\alpha \sum_{t=\tau+1}^T   \frac{t^\alpha-(\tau+\theta_b)^\alpha}{t^{2\alpha} - (\tau+\theta_b)^{2\alpha}} \\
      & =  - \Theta(T^{1-\alpha}) + \beta \tau^\alpha \sum_{t=\tau+1}^T   \frac{t^\alpha-(\tau+\theta_b)^\alpha}{t^{2\alpha} } \\
      & =  - \Theta(T^{1-\alpha}) + \beta \tau^\alpha \sum_{t=\tau+1}^T  \frac{1}{t^{\alpha} }  -  \beta \tau^\alpha \sum_{t=\tau+1}^T \frac{(\tau+\theta_b)^\alpha}{t^{2\alpha} }  \\
      & = - \Theta(T^{1-\alpha}) + \beta \tau^\alpha \Theta(T^{1-\alpha}) - \Theta(\tau^{2\alpha} T^{1-2\alpha}) \\
      & = \Theta(\tau^\alpha \Theta(T^{1-\alpha}) ),
\end{align*}
where we use $\tau = o(T)$ for the last equality. Recall that we are considering the case where $\tau \ge  \frac{c}{2\alpha} \ln^{\frac{1}{1-\alpha}} T$. Note that the above bound takes the smallest value when $\tau = \frac{c}{2\alpha} \ln^{\frac{1}{1-\alpha}} T$. This completes the proof of the theorem.

\subsection{Proof of Theorem~\ref{thm:BEregret}}

To analyze the BE algorithm we will, as a stepping stone, analyze a slightly more general policy where $n$ is chosen arbitrarily, but still sub-linearly in $T$, as follows.

\begin{proposition}\label{prop:Balanced-Exploration}
Consider a variant of Balanced-Exploration algorithm where $n$ is allowed to be chosen arbitrarily while ensuring that it is $o(T)$.  For each $\alpha$, there exists a constant $c_\alpha$ such that the regret under Balanced-Exploration policy satisfies the following: 
\begin{enumerate}
\item If $0< \alpha < 1$ then regret is $O(n^\alpha T^{1-\alpha}  + T e^{-c_\alpha n})$. 
\item If $\alpha = 1$ then regret is $O(n \ln T + T e^{-c_1 n}))$. 
\item If $\alpha >1$ then regret is $O(n^\alpha +  T e^{-c_\alpha n}))$. 
\end{enumerate}
\end{proposition}
We now prove this proposition, and later use it to prove the theorem. 
  
By Law of Total Expectation, we have
\begin{align}
\mathbb E[R_T] & =  \mathbb E[R_T| \hat a^* = a^*] \mathbb P(\hat a^* = a^*) +   \mathbb E[R_T|\hat a \neq \hat a^*] \mathbb P(\hat a^* = a^*) \nonumber  \\
& \le  \mathbb E[R_T|\hat a^* = a^*]  +   T\mathbb P(\hat a^* \neq a^*). \label{eq:RtotalExpectation}
\end{align}

We first obtain a bound on $\mathbb E[R_T|\hat a^* = a^*]$ and then on $\mathbb P(\hat a^* \neq a^*)$, from which the proposition would follow. 

From the definition of cumulative regret we have 
$$ \mathbb E[R_T|\hat a^* = a^*] =  \mathbb E[\Gamma^*_T] - \mathbb E[\Gamma_T|\hat a^* = a^*]$$

 We can lower-bound total rewards obtained by only counting rewards obtained during from time $\tau_n+1$ to $T$, i.e.,
$$\mathbb E[\Gamma_T|\hat a^* = a^*] \ge \mathbb E[\Gamma_{\text{exploit}}|\hat a^* = a^*],$$
where  $\Gamma_{\text{exploit}}$ represents cumulative rewards obtained during the exploitation phase.

Thus, we get

\begin{equation}\label{eq:RgivenOptimal}
\mathbb E[R_T|\hat a^* = a^*] \le \mathbb E[\Gamma^*_T] -  \mathbb E[\Gamma_{\text{exploit}}|\hat a^* = a^*].
\end{equation}


We now obtain a lower bound on $\E[\Gamma_{\text{exploit}}|\hat a^* = a^*]$. 
Note that $N_a(\tau_n) = n  + \theta_a $ for each arm $a$.  A lower bound on $\E[\Gamma_{\text{exploit}}|\hat a = \hat a^*, \tau_n]$ is obtained using an argument same as to that used for obtaining the lower bound on $\E[\Gamma^*]$ in Proposition \ref{prop:Oracle}, with $\theta^\alpha $ replaced with $\sum_{a\neq a^*} (n + \theta_a )^\alpha$ and looking at times $\tau_n+1$ to $T$ instead of times $1,\ldots,T$. Thus, we get

$$\mathbb E[\Gamma_{\text{exploit}}|\hat a^* = a^*,\tau_n]  \ge  \mu_{a^*} (T - \tau_n)  - \left(\sum_{a\neq a^*} (n + \theta_a )^\alpha\right) \sum_{k=\tau_n}^{T} \frac{1}{(k+\theta_{a^*})^\alpha}-1.
$$

Taking expectation w.r.t. $\tau_n$, we get
\begin{multline*} 
\mathbb E[\Gamma_{\text{exploit}}|\hat a^* = a^*]  \ge  \mu_{a^*} (T - \mathbb E[\tau_n]) \\ - \left(\sum_{a\neq a^*} (n + \theta_a )^\alpha\right) \mathbb E \left[ \sum_{k=\tau_n}^{T} \frac{1}{(k+\theta_{a^*})^\alpha} \right]-1.
\\
\ge \mu_{a^*} (T - \mathbb E[\tau_n]) -  \left(\sum_{a\neq a^*} (n + \theta_a )^\alpha\right) \left[ \sum_{k=1}^{T} \frac{1}{(k+\theta_{a^*})^\alpha} \right]-1.
\end{multline*}

Using the above bound and Part $1.$ of Proposition \ref{prop:Oracle} in \eqref{eq:RgivenOptimal} we obtain,

\begin{multline*}
\mathbb E[R| \hat a^* = a^*] \le T \mu_{a^*} - \mu_{a^*}\theta^\alpha\sum_{k=1}^T \frac{1}{(\mu_{a^*} k)^\alpha + \theta^\alpha} \\ -  \mu_{a^*} (T - \mathbb E[\tau_n])  + \left(\sum_{a\neq a^*} (n + \theta_a )^\alpha\right) \sum_{k=1}^{T} \frac{1}{(k+\theta_{a^*})^\alpha}+1. 
\end{multline*}
Thus, we obtain
$$\mathbb E[R| \hat a^* = a^*] \le \mu_{a^*} \mathbb E[\tau_n] - \mu_{a^*}\theta^\alpha\sum_{k=1}^T \frac{1}{(\mu_{a^*} k)^\alpha + \theta^\alpha}  + \left(\sum_{a\neq a^*} (n + \theta_a )^\alpha\right) \sum_{k=1}^{T} \frac{1}{(k+\theta_{a^*})^\alpha}+1. 
$$

We now show that $\mathbb E[\tau_n] = O(n) $. During exploration phase, the algorithm operates in $n$ cycles, where at the beginning of cycle $k$ the $N_a$ for each arm $a$ is equal to $k+\theta_a-1$, and it equals to $k+\theta_a$ at the end of the cycle. Thus, when arm $a$ is pulled, the probability that it obtains a unit reward is at least $\frac{(\theta_a + k-1)}{\sum_{b\in A}(\theta_b + k)} \mu_a$. Thus, it takes $O(1)$ expected number of attempts on an arm to obtain a unit reward in each cycle. Thus, to obtain $n$ rewards at all arms it takes $\mathbb E[\tau_n] = O(n) $ time.


Thus, for $0<\alpha<1$ we have
\begin{align*}
\mathbb E[R| \hat a^* = a^*]  &\le  \mu_{a^*} O(n) - \Omega(T^{1-\alpha}) + O(n^\alpha T^{1-\alpha})  \\
&  = O(n^\alpha T^{1-\alpha}) .
\end{align*}

Similarly  we obtain that $\mathbb E[R| \hat a^* = a^*] $ is $O(n \ln T)$ for $\alpha = 1$ and it is $O(n)$ for $\alpha>1$.  


Thus, the proposition would follow if we show that $\mathbb P(\hat a^* \neq a^*) \le e^{-c_\alpha n}$ for some positive constant $c_\alpha$. We show that below.  We start with special case where $\theta_a = 1$ for each $a$.

\begin{lemma}\label{lemma:conc_muhat}
Suppose $\theta_a = 1$ for each $a\in A$. Let $\delta = \min_{a\neq a^*} (\mu_{a^*} - \mu_a)$. For each arm $b$, there exists a constant $c_b$ independent of $n$ such that 
$$\mathbb P\left( \hat \mu_{b}(\tau_n)  > \mu_{b} + \frac{\delta}{2} \right) \le e^{-c_b n}.$$
Similarly, there exists a constant $c'_b$ independent of $n$ such that 
$$\mathbb P\left( \hat \mu_{a^*}(\tau_n)  < \mu_{a^*} -  \frac{\delta}{2} \right) \le e^{-c'_b n}.$$
\end{lemma}
To prove the lemma, note that for each small constant $\epsilon>0$ there exists an integer constant $k_\epsilon$ such that for each time $t$ after the $k_\epsilon\kth$ cycle, we have $(1-\epsilon)/m \le \lambda_b(t) \le (1+\epsilon)/m$ for each arm $b$. Thus, after a constant $k_{\frac{\delta}{4\mu_b}}$ number of pulls of arm $b$, we have that each pull of arm $b$ results into a success with probability no larger than $\mu_b(1+\frac{\delta}{4\mu_b})/m$ which equals $\frac{1}{m}(\mu_b + \frac{\delta}{4})$. Thus, when arm $b$ is pulled, time to each success is a Geometric random variable with rate less than or equal to $\frac{1}{m}(\mu_b + \frac{\delta}{4})$. Thus, the first part of the lemma follows from standard exponential concentration result for independent Geometric random variables \cite{DeZ98}. Second part of the lemma follows similarly.

Thus, the proposition follows for the case where $\theta_a = 1$ for each $a$. For general values of $\theta_a$ essentially the same argument applies by observing that for each small constant $\epsilon>0$ there exists an integer constant $k_\epsilon$ such that for each time $t$ after $k_\epsilon\kth$ cycle, we have $(1-\epsilon)/m \le \lambda_a(t) \le (1+\epsilon)/m$ for each arm $a$. Since $k_\epsilon$ does not depend on $n$, the concentration arguments above still hold.  This completes the proof for the proposition. 

Now, recall that in the statement of Theorem~\ref{thm:BEregret} where we set $n = w_T \ln T$. Since $w_k$ is $\omega(k)$, there exists $k$ such that $w_k \ge 2 c_\alpha$. Thus, $\mathbb P(\hat a^* \neq a^*) = O(1/T^2)$. This completes the proof of Theorem~\ref{thm:BEregret}.

\subsection{Proof of Theorem~\ref{thm:BEAEregret}}

We will prove the result for $\alpha <1$. The result for general $\alpha$ follows using essentially the same argument. Similar to the BE algorithm, the BE-AE algorithm can be thought of as containing exploration phase and exploitation phase. The exploration phase consists of times $t= 0\ldots \tilde t$ where $\tilde t = \max(t\le T: |A(t)| \ge 2)$, and the exploitation phase consists of times $t> \tilde t$. Let the arm active during the exploitation phase be denoted by $\hat a^*$. Then, similar to proof of Proposition~\ref{prop:Balanced-Exploration}, we have 

\begin{align}
\mathbb E[R_T]   
& \le  \mathbb E[R_T|\hat a^* = a^*]  +   T\mathbb P(\hat a^* \neq a^*) \\
& \le \sum_{a\neq a^*} \E[T_a(\tilde t)]  + \sum_{a\neq a^*} \E[ N_a(\tilde t)] T^{1-\alpha} +   T\mathbb P(\hat a^* \neq a^*) \\
& \le  \sum_{a\neq a^*} \E[T_a(T)] + \sum_{a\neq a^*} \E[T_a(T) + \theta_a] T^{1-\alpha} +  +   T\mathbb P(\hat a^* \neq a^*) \\
\end{align}

Thus, it is sufficient to show that $P(\hat a^* \neq a^*) = O(T^{-1})$ and that $\E[T_a(T)] = O(\ln T)$. In turn, it sufficient to show that $\P(\exists t \text{ s.t. }a^* \notin A(t))= O(T^{-1})$ and that $\E[T_a(T)] = O(\ln T)$. We do that below. We will use the following lemmas, proven in Section~\ref{sec:LemmasBEAEregret}.

\begin{lemma}\label{lemma:lambda_bound}
We have $\lambda_a(t) \ge c$ for each $t$ and each $a \in A(t)$. 
\end{lemma}

\begin{lemma}\label{lemma:Confidence1}
For each arm $a\in A$ we have
\begin{enumerate}
\item $\P(\exists t\le T \text{ s.t. } u_a(t) \le \mu_a) \le T^{-1}$ 
\item $\P(\exists t\le T \text{ s.t. } l_a(t) \ge \mu_a) \le T^{-1}$ 
\end{enumerate}
\end{lemma}
 
\begin{lemma}\label{lemma:Confidence2}
Let $\delta = \min_{a\neq a^*} (\mu_{a^*} - \mu_a)$. There exists a constant $\beta$ such that  we have 
\begin{enumerate} 
\item $\P(\exists t\le T \text{ s.t. } T_a(t) \ge \beta \ln T,   u_a(t) \ge \mu_a + \delta/2 ) \le T^{-1}$ 
\item $\P(\exists t\le T \text{ s.t. }  T_a(t) \ge \beta \ln T,  l_a(t) \le \mu_a - \delta/2) \le T^{-1}$ 
\end{enumerate}
\end{lemma}

\begin{lemma}\label{lemma:LotsofHits} 
Recall $\beta$ from Lemma \ref{lemma:Confidence2}. For a large enough positive constant $\gamma$ we have that for $t' =\gamma \ln T$ we have $\P(T_a(t') \le \beta \ln T, a\in A(t')) \le T^{-2}$ for each arm $a \in A$.
\end{lemma}

 Now, using union bound we get,
\begin{align*}
\P(\exists t \text{ s.t. }a^* \notin A(t))& \le \sum_{a\neq a^*}\P(\exists t  \text{ s.t. } u_{a^*}(t) < l_a(t)) \\
			& \le \sum_{a \neq a^*} \left(\exists t \text{ s.t. } \P(u_{a^*}(t) \le \mu_{a^*}) + \P(\exists t \text{ s.t. } l_a(t) \ge \mu_a) \right) \\
			& = O(1/T),
\end{align*}
where the last bound follows from Lemma~\ref{lemma:Confidence1}. 
Thus, it is now sufficient to show that $\E[T_a(T)]$ for each $a\neq a^*$ is $O(\ln T)$. Let $\gamma >0$ be a constant to be determined. Let $t' = \gamma \ln T$. 
We have,
\begin{align*}
\E[T_a(T)] & \le \E[T_a(T) | a\notin A(t')] + T \P(a\in A(t')) \\
		& \le t' + T \P(a \in A(t')) \\
		& = \gamma \ln T + T \P(a\in A(t'))
\end{align*}

Thus, we will be done if we show that $\P(a \in A(t'))= O(T^{-1})$ for a large enough $\gamma$. We do that below. By Law of Total Probability and the fact that $\P(\exists t \text{ s.t. } a^* \notin A(t)) = O(T^{-1})$ as shown above, we have
$$\P(a \in A(t'))  \le \P(a \in A(t'), a^* \in A(t')) + \P(a^* \notin A(t')) = \P(a \in A(t'), a^* \in A(t')) + O(T^{-1}).$$ 

Further, 
\begin{align*}
\P(a \in A(t'), a^* \in A(t')) &  \le  \P(T_a(t') \le \beta \ln T, a\in A(t')) + \P(T_{a^*}(t') \le \beta \ln T, a^* \in A(t')) \\ 
   & + \P(a \in A(t'), a^* \in A(t'),  T_a(t') \ge \beta \ln T, T_{a^*}(t') \ge \beta \ln T) \\
&  = O(1/T^2) + \P(a \in A(t'), a^* \in A(t'), T_a(t') \ge \beta \ln T, T_{a^*}(t') \ge \beta \ln T), 
\end{align*}
where the last equality follows from Lemma~\ref{lemma:LotsofHits}. We now provide a bound on the last term of the above inequality. Event $a,a^* \in A(t')$ implies that $u_a(t') < l_{a^*}(t')$. Thus, we get
\begin{align*}
& \P(a \in A(t'), a^* \in A(t'), T_a(t') \ge \beta \ln T, T_{a^*}(t') \ge \beta \ln T) \\ & \le  \P(u_a(t') < l_{a^*}(t'),  T_a(t') \ge \beta \ln T, T_{a^*}(t') \ge \beta \ln T), \\
&  \le \P(u_a(t') < \mu_a +\delta/2, T_a(t') \ge \beta \ln T) + \P( l_{a^*}(t') ,T_{a^*}(t') \ge \beta \ln T),\\
& = O(1/T),
\end{align*}
 
 where the last inequality follows from Lemma~\ref{lemma:Confidence2}.  Hence the result follows.

 \subsubsection{Proof of lemmas used in proof of Theorem~\ref{thm:BEAEregret}}\label{sec:LemmasBEAEregret}
 {\bf Lemma \ref{lemma:lambda_bound}.}
We have $\lambda_a(t) \ge c$ for each $t$ and each $a \in A(t)$. 
\begin{proof}
From the definition of the algorithm, we have $|S_a(t) - S_b(t)| \le 1$ for each $a,b \in A(t)$. Further, for each $a\in A(t)$ and $b \notin A(t)$ we have $S_a(t) \ge S_b(t) -1$. Let $b' \in \arg \max_a \theta_a$. Thus, for each $a\in A(t)$ we have
$$\lambda_a(t) = \frac{S_a(t) +\theta_a}{\sum_b (S_b(t) +\theta_b)} \ge  \frac{S_a(t) +\theta_a}{\sum_b (S_a(t) +1 +\theta_b)} \ge  \frac{S_a(t) +\theta_a}{m (S_a(t) +1 +\theta_{b'})} \ge \frac{\theta_a}{m(1 + \theta_{b'})}.$$
Hence, the lemma holds since $c = \min_a  \frac{\theta_a}{m(1 + \theta_{b'})}$.
\end{proof}
 
{\bf Lemma \ref{lemma:Confidence1}.}
For each arm $a\in A$ we have
\begin{enumerate}
\item $\P(\exists t\le T \text{ s.t. } u_a(t) \le \mu_a) \le T^{-1}$ 
\item $\P(\exists t\le T \text{ s.t. } l_a(t) \ge \mu_a) \le T^{-1}$ 
\end{enumerate}
 
 \begin{proof}
 We first prove the first part 1. We will use the following result, known as Freedman's inequality for martingales. 
 
\begin{theorem}[Freedman \cite{fre75}]
Let $(W_t, \mathcal F_t)_{i=0,..,T}$ be a real valued martingale. Let $(\xi_t, \mathcal F_t)_{t=0,..,T}$ be the sequence of corresponding martingale differences, i.e., $W_t = \sum_{i=0}^t \xi_t$, s.t.\ $\xi_0 = 0$. Let $V_k = \sum_{i=1}^{t} \E[\xi^2|\mathcal F_{i-1}]$. Suppose $\xi_t \le \epsilon$ for a some positive $\epsilon$. Then, the following holds for all positive $w$ and $v$. 
$$\P\left(\exists t \text{ s.t. } W_t \ge w \text{ and } V_t \le v \right) \le \exp\left(-\frac{w^2}{2(v + w \epsilon)}\right).$$
\end{theorem}

Let $M_0 = 0$ and for each $t \ge 1$ let $M_t = \mu_a T_a(t) -  \hat{\mu}_a(t)T_a(t)$. Let $\{\mathcal F_t\}_{t\ge0}$ represent the filtration where $\mathcal F_0 = \{\emptyset, \Omega\}$ and $\mathcal F_t$ captures what is known to the platform at each time $t$. Then, it is easy to check that $(M_t, \mathcal F_t)_{i=0,..,T}$ forms a martingale. Consider stopping times $\tau_k$ where $\tau_k = \inf \{t: T_a(t) = k\}$ if $T_a(T) \le k$ else $\tau_k = T$. Let $Y_0 = 0$ and $Y_k = M_{\tau_k}$. From Optional Sampling Theorem (see Chapter A14 in \cite{wil91}) we get that $(Y_k, \mathcal F_{\tau_k})_{k\ge 0}$ is a martingale.

We now provide bound on $\P(\exists t \text{ s.t. } u_a(t) \le \mu_a) $. Note that if $T_a(t) < 25 c^{-1} \ln T$ then $5\sqrt{ \frac{ \ln T}{cT_a(t)} } > 1 \ge \mu_a $ and $u_a(t) > \mu_a$.   Thus, we get,
\begin{align*}
\P(\exists t \text{ s.t. } u_a(t) \le \mu_a) 
				& = \P\left(\exists t \text{ s.t. } T_a(t) \ge 25 c^{-1} \ln T,  \hat{\mu}_a(t) + 5\sqrt{ \frac{ \ln T}{cT_a(t)} } \le \mu_a \right) \\
				& = \P\left(\exists t \text{ s.t. } T_a(t) \ge 25 c^{-1} \ln T,  M_t \ge 5 T_a \sqrt{ \frac{ \ln T}{cT_a(t)} }  \right) \\
				& = \P\left( \exists k\ge 25 c^{-1} \ln T \text{ s.t. } Y_k \ge 5\sqrt{  c^{-1} k\ln T}  \right)
\end{align*}

Let $D_k = Y_k - Y_{k-1} = \mu_a - \frac{X_{\tau_k}}{\lambda_a(\tau_k)} $. From Lemma~\ref{lemma:lambda_bound} we have $- c^{-1} \le D_k \le \mu_a \le 1$ for each $k$. 

From the definition of BE-AE algorithm, since ties are broken deterministically, we have that $I_t \in \mathcal{F}_{t-1}$. Thus, $\tau_{k} - 1$ is a stopping time. Thus, $\mathcal{F}_{\tau_k -1}$ is well defined. Now, $\E[D_k^2| F_{\tau_{k-1}}] = \E[ \E[D_k^2| F_{\tau_{k} - 1}] | F_{\tau_{k-1}}]$, where 
$$\E[D_k^2| F_{\tau_{k} - 1}] = \frac{\mu_a}{\lambda_a(\tau_k)} - \mu_a^2 \le \frac{1}{\lambda_a(\tau_k)} \le c^{-1}.$$  Thus for each $k$ we have $\sum_{i=1}^k \E[D_i^2| F_{\tau_{i-1}}] \le c^{-1}k$ with probability $t$. 

Fix a $k$ such that $25 c^{-1} \ln T \le k\le T$.  Using Freedman's inequality we get


 
				
\begin{align*}			
\P\left( \exists i  \text{ s.t. } Y_i \ge 5 \sqrt{  c^{-1} k\ln T}, \sum_{j=1}^i \E[D_j^2| F_{\tau_{j-1}}] \le c^{-1}k \right)
				& =  \exp\left(-\frac{25 k c^{-1} \ln T}{2(c^{-1}k +  5\sqrt{  c^{-1} k\ln T})}\right)  \\
				& \le \exp\left(-2 c^{-1}\ln T) \right) \\
				& \le T^{-2}
\end{align*}

Thus,
$$\P\left(  Y_k \ge 5 \sqrt{  c^{-1} k\ln T}, \sum_{j=1}^k \E[D_j^2| F_{\tau_{j-1}}] \le c^{-1}k \right)
	 \le T^{-2}.$$

But, as saw above,  $ \sum_{j=1}^k \E[D_j^2| F_{\tau_{j-1}}] \le c^{-1}k$ holds w.p.\ 1. Thus, 

$$\P\left(  Y_k \ge 5 \sqrt{  c^{-1} k\ln T} \right) \le T^{-2}.$$

Using union bound, we get 
$$\P\left( \exists k \ge 25 c^{-1} \ln T \text{ s.t. }  Y_k \ge 5 \sqrt{  c^{-1} k\ln T}  \right) \le T^{-1}.$$

Hence, we get $ \P(\exists t \text{ s.t. } u_a(t) \le \mu_a)  \le T{-1}$. Proof for part 2.\ is similar to above, except that we work with martingale $(-Y_k, \mathcal F_{\tau_k})_{k\ge 0}$ instead of $(Y_k, \mathcal F_{\tau_k})_{k\ge 0}$.
 \end{proof}

{\bf Lemma~\ref{lemma:Confidence2}.}
Let $\delta = \min_{a\neq a^*} (\mu_{a^*} - \mu_a)$. There exists a constant $\beta$ such that  we have 
\begin{enumerate} 
\item $\P(\exists t\le T \text{ s.t. } T_a(t) \ge \beta \ln T,   u_a(t) \ge \mu_a + \delta/2 ) \le T^{-1}$ 
\item $\P(\exists t\le T \text{ s.t. }  T_a(t) \ge \beta \ln T,  l_a(t) \le \mu_a - \delta/2) \le T^{-1}$ 
\end{enumerate}
 \begin{proof}
 We first prove the second part. Arguing along the lines similar to the proof of Lemma~\ref{lemma:Confidence1}, it is sufficient to show that there exists $\beta$ such that 
 
 $$\P(\exists k \ge \beta \ln T \text{ s.t. } Y_k \ge \delta k/2 -  5\sqrt{  c^{-1} k\ln T} ) \le T^{-2}.$$
For large enough $\beta$, for each $k \ge \beta \ln T$ we have $\delta k/2 -  5\sqrt{  c^{-1} k\ln T} \ge \delta k/4$. Further, for large enough $k$, using Freedman's inequality and the arguments similar to those in Lemma~\ref{lemma:Confidence1}, we get

$$
\P(Y_k \ge \delta k/4  ) 
		\le \exp( -\frac{\delta^2 k^2/16 }{2(c^{-1} k+ c^{-1}  \delta k/4)} ) = \exp(-c_1 k)$$
where $c_1>0$. For large enough $\beta$ we have $c_1k \ge 2 \ln T$ for each $k \ge \beta \ln T$, and thus $\P( Y_k \ge \delta k/4  ) \le T^{-2}$. The result then follows by using a union bound. The second part of the result follows in a similar fashion by using martingale $(-Y_k, \mathcal F_{\tau_k})_{k\ge 0}$ instead of $(Y_k, \mathcal F_{\tau_k})_{k\ge 0}$.
 \end{proof}
 
{\bf Lemma~\ref{lemma:LotsofHits}.}
Recall $\beta$ from Lemma \ref{lemma:Confidence2}. For a large enough positive constant $\gamma$ we have that for $t' =\gamma \ln T$ we have $\P(T_a(t') \le \beta \ln T, a\in A(t')) \le T^{-2}$ for each arm $a \in A$.
\begin{proof}
Let $\Delta = \min_a \mu_a$. Thus, $\P(X_t = 1) \ge c\Delta$ for each $t$ under the BE-AE algorithm. Thus, using standard Chernoff bound, for a large enough $\gamma$ we have 
$\P(\sum_{t=1}^{t'} X_t \le m \beta \ln T + m) \le e^{-2 \ln T}$. Since under BE-AE we have $|S_b(t) - S_{b'}(t)| \le 1$ for each $b,b' \in A(t)$, for a large enough $\gamma$, $\P(\exists a \in A(t') \text{ s.t. } S_a(t') \le \beta \ln T ) \le e^{-2 \ln T}$. Thus, $\P(\exists b \in A(t') \text{ s.t. } T_b(t') \le \beta \ln T ) \le e^{-2 \ln T}$. Hence the lemma holds. 
\end{proof}

\end{document}